\documentclass[accepted]{uai2022} 

\usepackage[american]{babel}

\usepackage{natbib} 
\usepackage{mathtools} 
\usepackage{booktabs} 
\usepackage{tikz} 

\usepackage[ruled,vlined,linesnumbered,noend]{algorithm2e}

\usepackage{enumitem}
\usepackage{graphicx}
\usepackage{amsfonts}
\usepackage{mathrsfs}
\usepackage{amsmath}
\usepackage{setspace}
\newtheorem{lemma}{Lemma}
\newtheorem{definition}{Definition}
\newtheorem{theorem}{Theorem}

\newenvironment{proof}{{\noindent\it Proof.}\quad}{\hfill $\square$ \par}
\newenvironment{proof_sketch}{{\noindent\it Proof Sketch.}\quad}{\hfill $\square$ \par}

\title{Efficient Resource Allocation with Fairness Constraints in Restless Multi-Armed Bandits}

\author[1]{\href{mailto:<dexunli.2019@phdcs.smu.edu.sg>?Subject=Your UAI 2022 paper}{Dexun Li}{}}
\author[1]{Pradeep Varakantham}
\affil[1]{%
    School of Computing and Information Systems\\
    Singapore Management University\\
    Singapore
}
  
\begin{document}
\maketitle

\begin{abstract}
  Restless Multi-Armed Bandits (RMAB) is an apt model to represent decision-making problems in public health interventions (e.g., tuberculosis, maternal, and child care), anti-poaching planning, sensor monitoring, personalized recommendations and many more. Existing research in RMAB has contributed mechanisms and theoretical results to a wide variety of settings, where the focus is on maximizing expected value. In this paper, we are interested in ensuring that RMAB decision making is also fair to different arms while maximizing expected value. In the context of public health settings, this would ensure that different people and/or communities are fairly represented while making public health intervention decisions. To achieve this goal, we formally define the fairness constraints in RMAB and provide planning and learning methods to solve RMAB in a fair manner. We demonstrate key theoretical properties of fair RMAB and experimentally demonstrate that our proposed methods handle fairness constraints without sacrificing significantly on solution quality. 
\end{abstract}

\section{Introduction}\label{sec:intro}
Picking the right time and manner of limited interventions is a problem of great practical importance in tuberculosis~\citep{mate2020collapsing}, maternal and child care~\citep{biswas2021learn,mate2021risk}, anti-poaching operations~\citep{qian2016restless},  cancer detection~\citep{lee2019optimal}, and many others. All these problems are characterized by multiple arms (i.e., patients, pregnant mothers, regions of a forest) whose state evolves in an uncertain manner (e.g.,  medication usage in the case of tuberculosis, engagement patterns of mothers on calls related to good practices in pregnancy) and threads moving to "bad" states have to be steered to "good" outcomes through interventions. The key challenge is that the number of interventions is limited due to a limited set of resources (e.g., public health workers, patrol officers in anti-poaching operations). Restless Multi-Armed Bandits (RMAB), a generalization of Multi-Armed Bandits (MAB) that allows non-active bandits to also undergo the Markovian state transition, has become an ideal model to represent the aforementioned problems of interest as it models uncertainty in arm transitions (to capture uncertain state evolution), actions (to represent interventions) and budget constraint (to represent limited resources). 

Existing work~\citep{mate2020collapsing,biswas2021learn,mate2021efficient} has focused on developing theoretical insights and practically efficient methods to solve RMAB. At each decision epoch, RMAB methods identify arms that provide the biggest improvement with an intervention. Such an approach though technically optimal can result in certain arms (or type of arms) getting starved for interventions.

In the case of interventions with regards to public health, RMAB algorithms focus interventions on the top beneficiaries who will improve the objective (public health outcomes) the most. This can result in certain beneficiaries never talking to public health workers and thereby moving to bad states (and potentially also impacting other beneficiaries in the same community) from where improvements can be minor even with intervention and hence never getting picked by RMAB algorithms. 
As shown in Fig.~\ref{fig:case}, when using the Threshold Whittle index approach proposed by~\citet{mate2020collapsing}, the arm activation probability is lopsided, with 30\% of arms getting activated more than 50 times and 50\% of the arms are never activated.
Such starvation of interventions can result in arms moving to a bad state from where interventions cannot provide big improvements and therefore there is further starvation of interventions for those arms. Such starvation can happen to entire regions or communities, resulting in lack of fair support for beneficiaries in those regions/communities. To avoid such cycles between bad outcomes, there is a need for RMAB algorithms to consider fairness in addition to maximizing expected reward when picking arms. 
Risk sensitive RMAB~\citep{mate2021risk} considers an objective that targets to reduce such starvation, however, they \emph{do not guarantee} that arms (or types of arms) are picked a minimum number of times. 

\begin{figure}[ht]
\centering
    \includegraphics[width=0.99\linewidth]{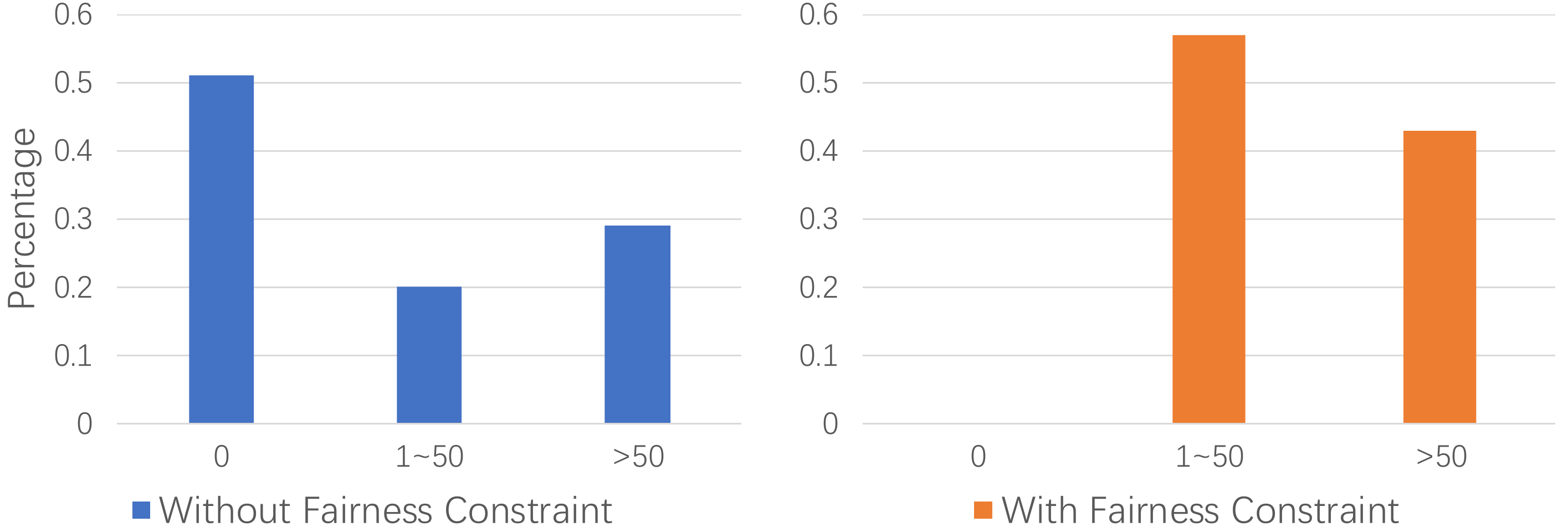}
    \caption{The x-axis is the number of times activated, and the y-axis is the percentage of each frequency range. We consider the RMAB given in Section~\ref{sec:prob}, with $k=10$, $N=100$, $T=1000$ and $L=50$, $\eta=2$. Left: the result of using the Whittle index algorithm without considering fairness constraints. Right: the result of  when considering fairness constraints. As can be noted, without fairness constraints in place, almost 50\% of the arms never get activated. }
    \label{fig:case}
\end{figure}

Recent work in Multi-Armed Bandits (MAB) has presented different notions of fairness. For example, 
\citet{li2019combinatorial} study a Combinatorial Sleeping MAB model with Fairness constraints, called CSMAB-F. The fairness constraints ensure a minimum selection fraction for each arm.
\citet{patil2020achieving} introduce similar fairness constraints in the stochastic multi-armed bandit problem, where they use a pre-specified vector to denote the guaranteed number of pulls.
\citet{joseph2016fairness} define fairness as saying that a worse arm should not be picked compared to a better arm, despite the uncertainty on payoffs.
\citet{chen2020fair} define the fairness constraint as a minimum rate that is required when allocating a task or resource to a user.
The above fairness definitions are relevant and we generalize from these to propose a fairness notion for RMAB. Unfortunately, approaches developed for fair MAB cannot be utilized for RMAB, due to uncertain state transitions with passive actions as well.

\paragraph{Contributions:} To the best of our knowledge, this is the first paper to consider fairness constraints in RMAB. Here are the key contributions:
\begin{itemize}
    \item We propose a fairness constraint wherein for any arm (or more generally, for a type of arm), we require that the number of decision epochs since the arm (or the type of arm) was activated last time is upper bounded. This will ensure that every arm (or type of arm) gets activated a minimum number of times, thus generalizing on the fairness notions in MAB described earlier. 
    \item We provide a modification to the Whittle index algorithm that is scalable and optimal while being able to handle both finite and infinite horizon cases. 
    We also provide a model-free learning method to solve the problem when the transition probabilities are not known beforehand.
    \item Experiment results on the generated dataset show that our proposed approaches can achieve good performance while still satisfying the fairness constraint.
\end{itemize}

\section{Problem Description}
\label{sec:prob}

In this section, we formally introduce the RMAB problem.
There are $N$ independent arms, each of which evolves according to an associated Markov Decision Process (MDP). An MDP is characterized by a tuple $\{\mathcal{S}, \mathcal{A}, \mathcal{P}, r\}$, where $\mathcal{S}$ represents the state space, $\mathcal{A}$ represents the action space, $\mathcal{P}$ represents the transition function, and $r$ is the state-dependent reward function.
Specifically, each arm has a binary-state space: $1$ ("good") and $0$ ("bad"), with action-dependent transition matrix $\mathcal{P}$ that is potentially different for each arm. Let $a_t^i\in \{ 0,1 \}$ denote the action taken at time step $t$ for arm $i$, and $a_t^i=1 (a_t^i=0)$ indicates an active (passive) action for arm $i$. Due to limited resources, at each decision epoch, the decision-maker can activate (or intervene on) at most $k$ out of $N$ arms and receive reward accrued from all arms determined by their states. $\sum_{i=1}^N a_{t}^i = k $ describes this limited resource constraint. Figure~\ref{fig:multi_channel} provides an example of an arm in RMAB. 

\begin{figure}[ht]
\centering
    \includegraphics[width=0.95\linewidth]{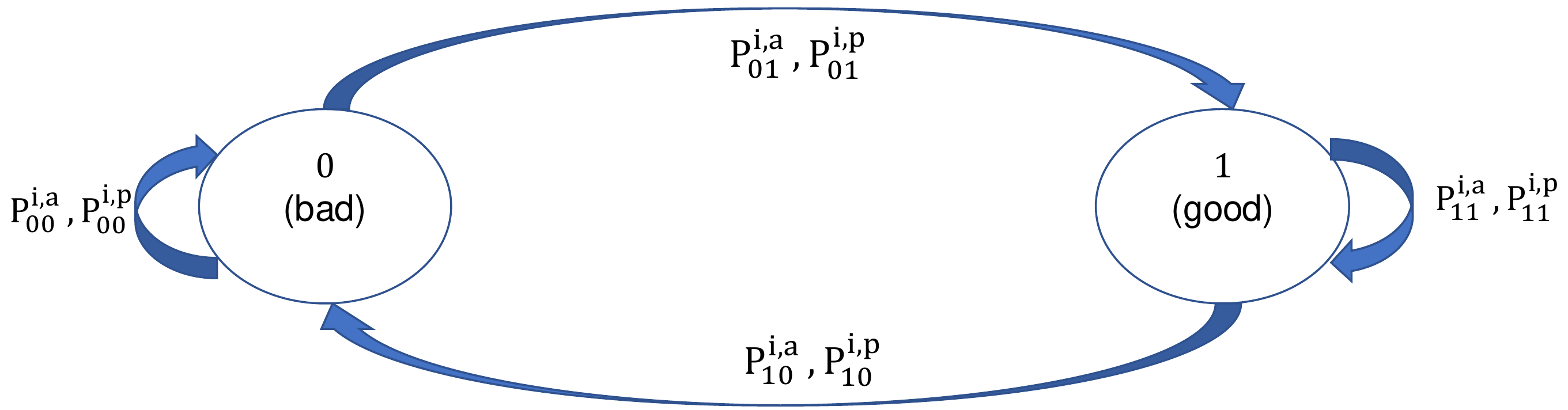}
    \caption{$a$ and $p$ denote the active and passive actions on arm $i$ respectively.  $P_{s,s^\prime}^{i,a}$ and $P_{s,s^\prime}^{i,p}$ are the transition probabilities from state $s$ to state $s^\prime$ under action $a$ and $p$ respectively for arm $i$.}
    \label{fig:multi_channel}
\end{figure}

The state of arm $i$ evolves according to the transition matrix $P_{s,s^\prime}^{a,i}$ for the active action and $P_{s,s^\prime}^{p,i}$ for the passive action. We follow the setting in~\citet{mate2020collapsing}, when the arm $i$ is activated, the latent state of arm $i$ will be fully observed by the decision-maker. The states of passive arms are unobserved by the decision-maker.

When considering such partially observable problem, it is sufficient to let the MDP state be the belief state: the probability that the arm is in the "good" state. We need to keep track of the belief state on the current state of the unobserved arm. This can be derived from the decision-maker's partial information which is encompassed by the last observed state and the number of decision time steps since the last activation of the arm. Let $\omega^i_s(u)$ denote the belief state, i.e., the probability that the state of arm $i$ is $1$ when it was activated $u$ time steps ago with the observed state $s$. 
The belief state in next time step can be obtained by solving the following recursive equations:
\begin{equation}\label{eq:1}
\resizebox{\linewidth}{!}{$
    \displaystyle
\begin{aligned}
  \omega^i_s(u+1) &= 
    \begin{cases}
      \omega^i_s(u) P_{1,1}^{p,i}+(1-\omega^i_s(u))P_{0,1}^{p,i}  &  \text{ passive}\\
      P_{s^\prime, 1}^{i,a}&  \text{ active}\\
    \end{cases}
\end{aligned}
$}
\end{equation}
Where $s^\prime$ is the new state observed for arm $i$ when the active action was taken. The belief state can be calculated in closed form with the given transition probabilities. We let $\omega=\omega^i_s(u+1)$ for ease of explanation when there is no ambiguity.

A policy $\pi$ maps the belief state vector $\Omega_t=\{\omega^1_t,\cdots,\omega^N_t \}$ at each time step $t$ for all arms to the action vector, $a_t=\{ 0,1\}^N$. Here $\omega^i_t$ is the belief state for arm $i$ at time step $t$. We want to design an optimal policy to maximize the cumulative long-term reward over all the arms. One widely used performance measure is the expected discounted reward over the horizon $T$ : $$\mathbb{E}_{\pi}[\sum_{t=1}^T  \beta^{t-1}R_t(\Omega_t,\pi(\Omega_t))|\Omega_0]$$ 
Here $R_t(\Omega_t,\pi(\Omega_t))$ is the reward obtained in slot $t$ under action $a_t=\pi(\Omega_t)$ determined by policy $\pi$, $\beta$ is the discount factor.
As we discussed in the introduction, in addition to maximizing the cumulative reward, ensuring fairness among the arms is also a key design concern for many real-world applications. In order to model the fairness requirement, we introduce constraints that ensure that any arm (or kind of arms) is activated at least $\eta$ times during any decision interval of length $L$.
The overall optimization problem corresponding to the problem at hand is thus given by: 
\begin{equation}
\resizebox{\linewidth}{!}{$
    \displaystyle
    \begin{aligned}
    \underset{\pi}{\text{maximize  }} &\mathbb{E}_{\pi}[\sum_{t=1}^T  \beta^{t-1}R_t(\Omega_t,\pi(\Omega_t))|\Omega_0]\\
    \text{subject to } &\sum_{i}^N a_t^i=k,\forall t\in \{ 1,\dots,T \}\\
    &\sum_{t=u}^{u+L}a_{t}^i\geq \eta \quad \forall u\in \{ 1,\dots,T-L \}, \forall i \in \{ 1, \dots, N \}.
    \end{aligned}
    $}
\end{equation}
$\eta$ is the minimum number of times an arm should be activated in a decision period of length $L$. The strength of fairness constraints is thus governed by the combination of $L$ and $\eta$. Obviously, this requires $k\times L>N\times(\eta-1)$ as the fairness constraint should meet the resource constraint. This fairness problem can be formulated at the level of regions/communities by also summing over all the arms, $i$ in a region in the second constraint, i.e.,
$$\sum_{i \in r} \sum_{t = u}^{u+L} a_t^i \geq \eta$$
Our approaches with a simple modification are also applicable to this fairness constraint at the level of regions/communities.

\section{Background: Whittle Index}

In this section, we describe the Whittle Index algorithm~\citep{Whittle1988restless} to solve RMAB. This algorithm at every time step, computes index values (Whittle Index values) for every arm and then activates the arms that have the top "$k$" index values. Whittle index quantifies how appealing it is to activate a certain arm. This algorithm provides optimal solutions if the underlying RMAB satisfies the indexability property, defined in Definition~\ref{def:Indexability}. 

Formally\footnote{Since we will only be talking about one arm at a time step, we will abuse the notation by not indexing belief, action and value function with arm id or time index.}, the Whittle index of an arm in a belief state $\omega$ (i.e., the probability of good state 1) is the minimum subsidy $\lambda$ such that it is optimal to make the arm passive in that belief state. Let $V_{\lambda, T}(\omega)$ denote the value function for the belief state $\omega$ over a horizon $T$. Then it could be written as
\begin{equation}
     V_{\lambda,T}(\omega) = \max \{ V_{\lambda,T}(\omega; a=0),V_{\lambda,T}(\omega; a=1)\},
\end{equation}
where $V_{\lambda,T}(\omega; a=0)$ and $V_{\lambda,T}(\omega; a=1)$ denote the value function when taking passive and active actions respectively at the first decision epoch followed by optimal policy in the future time steps. Because the expected immediate reward is $\omega$ and subsidy for a passive action is $\lambda$, we have the value function for passive action as:
\begin{align}\label{eq:a=0}
            V_{\lambda,T}(\omega,a=0)&=\lambda+ \omega+ \beta V_{\lambda,T-1}(\tau^1(\omega)),
\end{align}
where $\tau^1(\omega)$ is the $1$-step belief state update of $\omega$ when the passive arm is unobserved for another $1$ consecutive slot (see the update rule in Eq.~\ref{eq:1}). Note that $\omega$ is also the expected reward associated with that belief state. For an active action, the immediate reward is $\omega$ and there is no subsidy. However, the actual state will be known and then evolve according to the transition matrix for the next step:
\begin{align}
\label{eq:a=1}
    V_{\lambda,T}(\omega,a=1)=\omega+ & \beta( \omega V_{\lambda,T-1}(P_{1,1}^a)+\nonumber\\
    & (1-\omega)V_{\lambda,T-1}(P_{0,1}^a)).
\end{align}

\begin{definition}\label{def:Indexability}
An arm is indexable if the passive set under the subsidy $\lambda$ given as $\mathcal{P}_\lambda=\{\omega: V_{\lambda,T}(\omega,a=0)\geq V_{\lambda,T}(\omega,a=1) \}$
monotonically increases from $\emptyset$ to the entire state space as $\lambda$ increases from $-\infty$ to $\infty$. The RMAB is indexable if every arm is indexable.
\end{definition}

Intuitively, this means that if an arm takes passive action with subsidy $\lambda$, it will also take passive action if $\lambda^\prime > \lambda$.  Given the \textit{indexability},  $W_T(\omega)$ is the least subsidy, $\lambda$ that makes it equally desirable to take active and passive actions.
\begin{equation}\label{eq:Whittle_index_eq}
    W_T(\omega) = \underset{\lambda}{\inf}\{ \lambda: V_{\lambda,T}(\omega;a=1)\leq V_{\lambda,T}(\omega;a=0) \}
\end{equation}

\begin{definition}
A policy is a threshold policy if there exists a threshold $\lambda_{th}$ such that the action is passive $a=0$ if $\lambda>\lambda_{th}$ and $a=1$ otherwise.
\end{definition}
Existing efficient methods for solving RMABs derive these threshold policies.

\section{Fairness in RMAB}

The key advantage of a Whittle index based approach is scalability without sacrificing solution quality.  In this section, we provide Whittle index based approaches to handle fairness constraints under known and unknown transition models, with both infinite and finite horizon settings. We specifically consider partially observable settings\footnote{We also provide a discussion about fully observable  setting in the appendix}. 

\subsection{Infinite Horizon}

When we need to consider the partial observability of the state of the RMAB problem, it is sufficient to let the MDP state be the belief state: the probability that the arm is in the "good" state~\citep{kaelbling1998planning}. As a result, the partially observable RMAB has a large number of belief states~\citep{mate2020collapsing}. 

Recall that the definition of the Whittle index $W_T(\omega)$ of belief state $\omega$ is the smallest $\lambda$ s.t. it is optimal to make the arm passive in the current state. We can compute the Whittle index value for each arm, and then rank the index value of all $N$ arms and select top $k$ arms at each time step to activate. With fairness constraints, the change to the approach is minimal and intuitive. \textbf{\textit{The optimal policy is to choose the arms with the top "k" index values until a fairness constraint is violated for an arm. In that time step, we replace the last arm in top-$k$ with the arm for which fairness constraint is violated.}} We show that this simple change works across the board for the infinite and finite horizon, fully and partially observable settings. We provide the detailed algorithm in Algorithm~\ref{al:alg1} and also provide sufficient conditions under which the Algorithm~\ref{al:alg1} is optimal.

We now provide the expression for $\lambda$. $V_{\lambda, \infty}(\omega)$ denotes the value that can be accrued from a single-armed bandit process with subsidy $\lambda$  over infinite time horizon ($T\rightarrow \infty$) if the belief state is $\omega$. Therefore, we have:
\begin{equation}\label{eq:value_update}
\resizebox{\linewidth}{!}{$
    \displaystyle
\begin{aligned}
  V_{\lambda,\infty}(\omega) &= \max
    \begin{cases}
      \lambda+\omega +\beta  V_{\lambda,\infty}(\tau^1(\omega)) & \text{passive}\\
       \omega+\beta\left(\omega V_{\lambda,\infty}(P_{1,1}^a)+(1-\omega) V_{\lambda,\infty}(P_{0,1}^a)\right) & \text{active}
    \end{cases}
\end{aligned}
$}
\end{equation}

For any belief state $\omega$, the $u$-steps belief update $\tau^u(\omega)$ will converge to $\omega^{\ast}$ as $u\rightarrow \infty$, where $\omega^{\ast}=\frac{P_{0,1}^p}{1+P_{0,1}^p-P_{1,1}^p}$. It should be noted that this convergence can happen in two ways depending on the state transition patterns:
\begin{itemize}
\item{Case 1}: Positively correlated channel $(P_{1,1}^{p}\geq P_{0,1}^{p})$.

The belief update process is shown in Figure~\ref{fig:belief_update}.
We can see that for the positively correlated case, they have a monotonous belief update process. 

\begin{figure}[ht]
    \centering
    \includegraphics[width=0.95\linewidth]{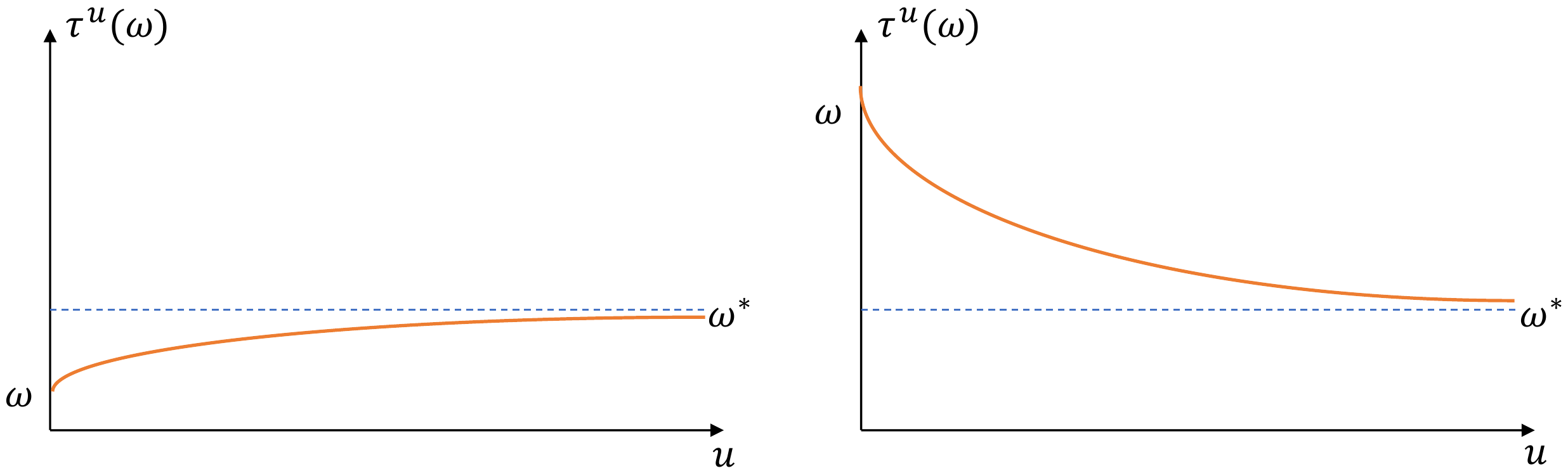}
    \caption{The $u$-step belief update of an unobserved arm $(P_{1,1}^p\geq P_{0,1}^p)$}
    \label{fig:belief_update}
\end{figure}

We first consider the \emph{non-increasing belief process} as indicated in the right graph. Formally, for $\forall u\in\mathbb{N}^+$, we have $\omega(u)\geq\omega(u+1)$ if the initial belief state $\omega$ is above the convergence value. Similarly, for the \emph{increasing belief process} shown in the left graph, we have the initial belief state $\omega<\omega^\ast$.

\item{Case 2}: Negatively correlated channel $(P_{1,1}^{p}< P_{0,1}^{p})$.
\begin{figure}[ht]
    \centering
    \includegraphics[width=0.95\linewidth]{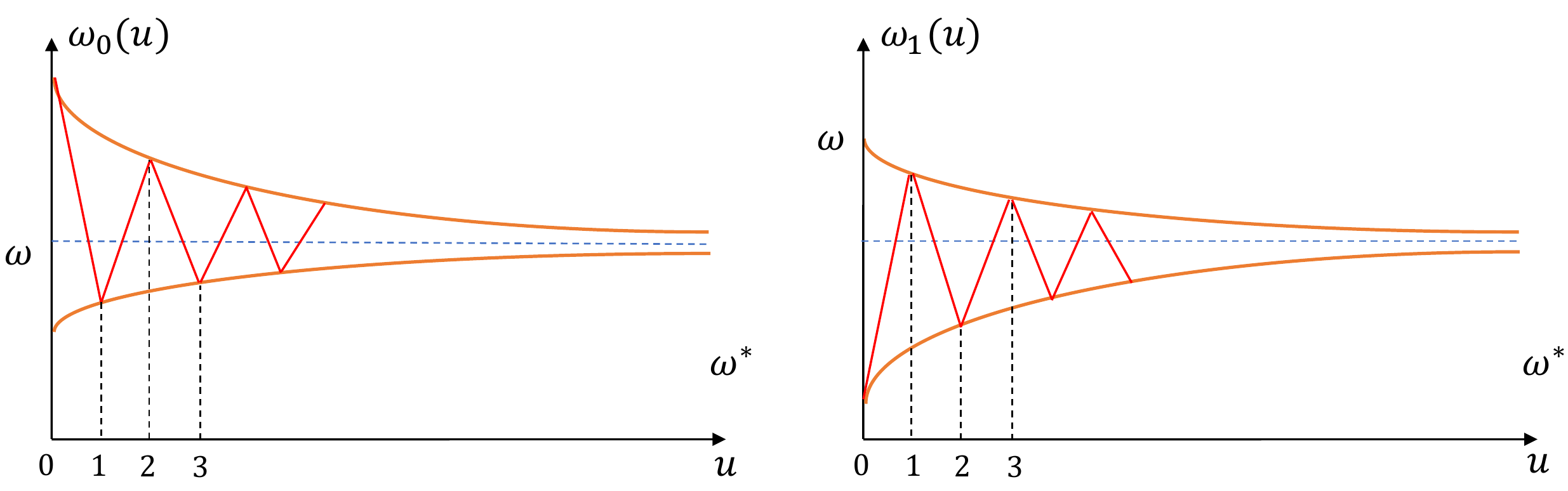}
    \caption{The $u$-step belief update of an unobserved arm $(P_{1,1}^p < P_{0,1}^p)$}
    \label{fig:belief_update2}
\end{figure}

The belief state converges to $\omega^{\ast}$ from the opposite direction as shown in Figure~\ref{fig:belief_update2}. This case has similar properties and is less common in the real world because it is more likely to remain in a good state than to move from a bad state to a good state. Therefore, we omit the lengthy discussion.
\end{itemize}

The belief state transition patterns are of particular importance because in proving optimality of Algorithm~\ref{al:alg1}, the belief evolution pattern for the arm (whose fairness constraint will be violated) plays a crucial role.

\begin{algorithm}[ht]
\caption{Fair Whittle Thresholding (FaWT)}
\label{al:alg1}
\SetAlgoLined
\KwIn{Transition matrix $\mathcal{P}$, fairness constraint, $\eta$ and $L$, set of belief states $\{\omega^1,\dots,\omega^N\}$, $k$}
\For{each arm i in $1$ to $N$}{
            Compute the corresponding Whittle index $TW(\omega^i)$ under the infinite horizon using the \textit{Forward and Reverse Threshold} policy\;
            \If{the activation frequency $\eta$ for arm $i$ will not be satisfied at the end of the period of length $L$}{
                Add arm $i$ to the action set $\phi$\;
                $k = k - 1$\;
            }
            
            \If{finite horizon} {
                Compute the the index value $W_1(\omega^i)$\;

                Compute the Whittle index $W_T(\omega^i)$ using Equation~\ref{eq:finite_Whittle}\;
            }

}
Add arms with top k highest $TW(\cdot)$ (for infinite horizon case) or $W_T(\cdot)$ (for finite horizon case) values to the action set $\phi$
Decrease the residual time horizon by $T=T-1$\;
\KwOut{Action set $\phi$}

\end{algorithm}

\begin{theorem}\label{thm:FC_infinite}
For infinite time horizon ($T\rightarrow\infty$) RMAB with Fairness Constraints governed by parameters $\eta$ and $L$, Algorithm~\ref{al:alg1} ( i.e., activating arm $i$ at the end of the time period when its fairness constraint is violated) is optimal:
\begin{enumerate}[leftmargin=*]
    \item For $\omega^i \leq \omega^*$ (\emph{increasing belief process}), if
    \begin{align}
(P_{1,1}^{i,p}-P_{0,1}^{i,p})\left(1+\frac{\beta \Delta_3}{1-\beta} \right)&{\left(1-\beta(P_{1,1}^{i,a}-P_{0,1}^{i,a})\right)} \nonumber\\
& \leq (P_{1,1}^{i,a} - P_{0,1}^{i,a})
\end{align}
 $\Delta_3 = \min \{(P_{1,1}^{i,p}-P_{0,1}^{i,p}),(P_{1,1}^{i,a}-P_{0,1}^{i,a})\}$. 

\item For $\omega^i \geq \omega^\ast$ (\emph{non-increasing belief process}), if:
\begin{align}
(P_{1,1}^{i,p}-&P_{0,1}^{i,p})(1-\beta)\Delta_1 \geq \nonumber\\
&(P_{1,1}^{i,a} - P_{0,1}^{i,a})\left(1-\beta(P_{1,1}^{i,a} - P_{0,1}^{i,a})\right)
\end{align}
 $\Delta_1 = \min \{ 1, 1+\beta(P_{1,1}^{i,p}-P_{0,1}^{i,p})-\beta(P_{1,1}^{i,a}-P_{0,1}^{i,a}) \}$

\end{enumerate}
\end{theorem}

\begin{proof_sketch}
Consider an arm $i$ that  has not been activated for $L-1$ time slots. In such a case, Algorithm~\ref{al:alg1} will select arm $i$ to activate in the next time step $t=L$. Define the intervention effect of activating arm $i$ as $$V_{\lambda,\infty}(\omega,a=1)-V_{\lambda,\infty}(\omega,a=0)$$
Following standard practice and for notational convenience, we do not index the intervention effect and value functions with $i$. 
Due to independent evolution of arms, moving active action of arm $i$ does not result in a greater value function for other arms according to the Whittle index algorithm, thus it suffices to only consider arm $i$. Here is the proof flow:\\
\noindent (1) Algorithm~\ref{al:alg1} optimality requires that the intervention effect at time step $t=L-1$ is smaller than intervention effect at $t=L$. Optimality can be established by requiring the partial derivative of the intervention effect w.r.t. time step $t$ is greater than 0.\\
(2) However, computing this partial derivative $\frac{\partial(V_{\lambda,\infty}(\omega,a=1)-V_{\lambda,\infty}(\omega,a=0)) }{\partial t}$ is difficult because value function expression is complex. We use chain rule to get:
$$ \frac{\partial(V_{\lambda,\infty}(\omega,a=1)-V_{\lambda,\infty}(\omega,a=0)) }{\partial \omega} \cdot \frac{\partial(\omega)}{\partial(t)}$$
(3) The sign of second term, $\frac{\partial \omega}{\partial t}$ is based on the belief state transition pattern described before this theorem. We then need to consider the sign of the first term, $\frac{\partial(V_{\lambda,\infty}(\omega,a=1)-V_{\lambda,\infty}(\omega,a=0)) }{\partial \omega}$.\\
(4) We can compute this by deriving the bound on $V_{\lambda,\infty}(\omega_1)-V_{\lambda,\infty}(\omega_2), \forall \omega_1, \omega_2$ as well as bounds on $\frac{\partial V_{\lambda,\infty}(\omega)}{\partial \omega}$. Detailed proof in appendix.\end{proof_sketch}

\begin{figure}
    \centering
    \includegraphics[width=0.95\linewidth]{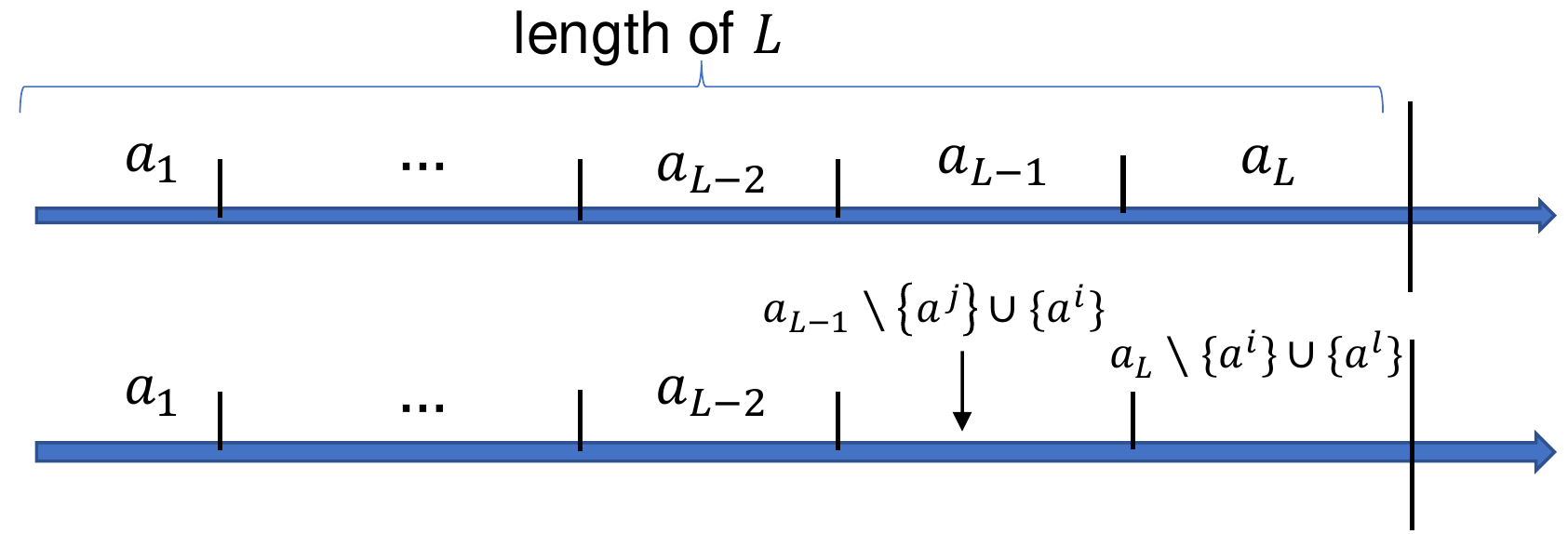}
    \caption{The action vector for RMAB is $a_t$ at time step $t$. Then we move the action $a^i$ that satisfies fairness constraint to earlier slot and replace $k$-th ranked action $a^j$. Action $a^l$ is then added according to the index value at the end.}
    \label{fig:fully_infinite}
\end{figure}

\subsection{Finite Horizon} 
In this part, we demonstrate that the mechanism developed for handling fairness in the infinite horizon setting can also be applied to the finite horizon setting. In showing this, we address two key challenges: 
\begin{enumerate}
\item  Computing the Whittle index under the finite horizon setting in a scalable manner. 
\item  Showing that Whittle index value reduces as residual horizon decreases. This will assist in showing that it is optimal to activate the fairness violating arm at the absolute last step where a violation will happen and not earlier; 
\end{enumerate}

It is costly to compute the index under the finite horizon setting -- $O(|\mathcal{S}|^kT)$ time and space complexity~\citep{hu2017asymptotically}. Therefore, we 
take advantage of the fact that the index value has an upper and lower bound, and it will converge to the upper bound as the time horizon $T\rightarrow \infty$. Specifically, we use an appropriate functional form to approximate the index value. To do this, we first show gradual Index decay ($\lambda_T > \lambda_{T-1} > \lambda_0$) by improving on the Index decay ($\lambda_T > \lambda_0$) introduced in ~\citep{mate2021efficient}. 
\begin{theorem}\label{thm:decay}
For a finite horizon $T$, the Whittle index $\lambda_T$ is the value that satisfies the equation $V_{\lambda_T,T}(\omega,a=0)= V_{\lambda_T,T}(\omega,a=1)$ for the belief state $\omega$. Assuming indexability holds, the Whittle index will decay as the value of horizon $T$ decreases: $\forall T>1: \lambda_{T+1}>\lambda_{T}>\lambda_0 = 0$.
\end{theorem}
\begin{proof_sketch}
 We can calculate $\lambda_0$ and $\lambda_1$ by solving equation $V_{\lambda_0,0}(\omega,a=0)=V_{\lambda_0,0}(\omega,a=1)$ and $V_{\lambda_1,1}(\omega,a=0)=V_{\lambda_1,1}(\omega,a=1)$ according to Eq.~\ref{eq:a=0} and Eq.~\ref{eq:a=1}. We can then derive $\lambda_t>\lambda_{t-1}$ by obtaining $\frac{\partial \lambda_t}{\partial t}>0$ for $\forall t>1$ through induction method. The detailed proof can be found in the appendix.
\end{proof_sketch}

We can easily compute $\lambda_0$, $\lambda_1$, and we have $\forall T>1: \lambda_{T+1}>\lambda_{T}>\lambda_0 = 0$ according to Theorem~\ref{thm:decay}, and $\underset{T\rightarrow \infty}{\lim}\lambda_{T}\rightarrow TW(\omega)$, where $TW(\omega)$ is the Whittle index value for state $\omega$ under infinite horizon. Hence, we can use a sigmoid curve to approximate the index value. One common example of a sigmoid function is the logistic function. This form is also used by~\citet{mate2021efficient}. Specifically, we let
\begin{equation}\label{eq:finite_Whittle}
    {\displaystyle  W_T(\omega)={\frac {A}{1+e^{-k T}}}+C,}
\end{equation}
where ${\displaystyle A}$ and ${\displaystyle \frac{A}{2}+C}$ are the curve's bounds; ${\displaystyle k}$ is the logistic growth rate or steepness of the curve. Recall that the definition of the Whittle index $W_T(\omega)$ of belief state $\omega$ is the smallest $\lambda$ s.t. it is optimal to make the arm passive in the current state. We have $W_0(\omega)=0$ and $W_1(\omega)=\beta(\omega(P_{1,1}^a-P_{1,1}^p)+(1-\omega)(P_{0,1}^a-P_{0,1}^p))$, and $W_{\infty}(\omega)=TW(\omega)$. By solving these three constraints, we can get the three unknown parameters, $$C=-TW(\omega),A=2TW(\omega),$$
\begin{equation*}
\resizebox{\linewidth}{!}{$
    k=-\log (\frac{2TW(\omega)}{\beta(\omega(P_{1,1}^a-P_{1,1}^p)+(1-\omega)(P_{0,1}^a-P_{0,1}^p))+TW(\omega)}-1)
    $}
\end{equation*}

Algorithm~\ref{al:alg1} shows how to use $W_T(\omega)$ in considering fairness constraint under the finite horizon setting. Next, we show that like in the infinite horizon case, value function and Whittle index decay over time in the case of the finite horizon. 

\begin{theorem}\label{thm:FC_finite}
Consider the finite horizon RMAB problem with fairness constraint. Algorithm~\ref{al:alg1} (activating arm $i$ at the end of the time period when its fairness constraint is violated) is optimal:
\begin{enumerate}[leftmargin=*]
    \item When $\omega^i\leq \omega^\ast$ (increasing belief process), if
\begin{align}
    (P_{1,1}^{i,p}-P_{0,1}^{i,p})&\left(\Delta_4\beta \sum_{t=0}^{T-2}[\beta^t]+1\right) \leq \nonumber\\
    &(P_{1,1}^{i,a} - P_{0,1}^{i,a})\sum_{t=0}^{T-2}[\beta^t(P_{1,1}^{i,a}-P_{0,1}^{i,a})^t]
\end{align}
 $\Delta_4= \min \{ (P_{1,1}^{i,p}-P_{0,1}^{i,p}), (P_{1,1}^{i,a}-P_{0,1}^{i,a})\}$, and $T$ is the residual horizon length.
\item When $\omega^i\geq \omega^\ast$ (non-increasing belief process), if
\begin{align}
    (P_{1,1}^{i,p}-P_{0,1}^{i,p})&\left(\Delta_2\beta\sum_{t=0}^{T-2}[\beta^t(P_{1,1}^{i,a}-P_{0,1}^{i,a})^t]+1\right) \geq \nonumber\\ 
    &(P_{1,1}^{i,a} - P_{0,1}^{i,a})\sum_{t=0}^{T-2}\beta^t
\end{align}
 $\Delta_2= \min \{ (P_{1,1}^{i,p}-P_{0,1}^{i,p}), (P_{1,1}^{i,a}-P_{0,1}^{i,a})\}$.
\end{enumerate}
\end{theorem}
\begin{proof_sketch}
The proof is similar to the infinite horizon case (detailed in Appendix).
\end{proof_sketch}

\subsection{Uncertainty in Transition Matrix}
In most real-world applications~\citep{biswas2021learn}, there may not be adequate information about all the state transitions. In such cases, we don't know how likely a transition is and thus, we won't be able to use the Whittle index approach directly. We provide a mechanism to apply the Thompson sampling based learning mechanism for solving RMAB problems without prior knowledge and where it is feasible to get learning experiences. Thompson sampling~\citep{thompson1933likelihood} is an algorithm for online decision problems, and can be applied in MDP~\citep{gopalan2015thompson} as well as Partially Observable MDP~\citep{meshram2016optimal}. 
In Thompson sampling, we initially assume that arm has a prior Beta distribution in the transition probability according to the prior knowledge (if there is no prior knowledge available, we assume a prior $Beta(1,1)$ as this is the uniform distribution on $(0,1)$). We choose Beta distribution because it is a convenient and useful prior option for Bernoulli rewards~\citep{agrawal2012analysis}.  

In our algorithm, referred to as FaWT-U and provided in \ref{al:alg2},  at each time step, we sample the posterior distribution over the parameters, and then use the Whittle index algorithm to select the arm with the highest index value to play if the fairness constraint is not violated. We can utilize our observations to update our posterior distribution, because playing the selected arms will reveal their current state.  Then, the algorithm takes samples from the posterior distribution and repeats the procedure again.

\begin{algorithm}[th]
\caption{Fair Whittle Thresholding with Uncertainty in transition matrix(FaWT-U)}
\label{al:alg2}
\SetAlgoLined
\KwIn{Posterior $Beta$ distribution over the transition matrix $\mathcal{P}$, fairness constraint, $\eta$ and $L$, set of belief states $\{\omega^1,\dots,\omega^N\}$, budget $k$}
\For{each arm i in $1$ to $N$}{
            Sample the  transition probability parameters independently from posterior\;
            Compute Whittle indices based on the transition matrix and belief state\;
}
            \If{the activation frequency $\eta$ for arm $i$ is not satisfied at the end of the period of length $L$}{
                Add arm $i$ to the action set $\phi$\;
                $k = k - 1$\;
            }
Add the arms with top $k$ index value into $\phi$\;
Play the selected arms and receive the observations\;
Update the posterior distribution\;
\KwOut{Action set $\phi$ and updated posterior distribution over parameters}
\end{algorithm}
We employ the sampled transition probabilities and belief states $\{\omega^1,\dots,\omega^N \}$, as well as the residual time horizon $T$ as the input to the Whittle index computation (Line 3 in Algorithm~\ref{al:alg2}).

\subsection{Unknown Transition Matrix}

We now tackle the second challenge mentioned, in which the transition matrix is completely unknown. In this case, we can take advantage of the model-free learning method to avoid directly using the whittling index policy.

Q-Learning is most commonly used to solve the sequential decision-making problem, which was first introduced by~\citet{watkins1992q} as an early breakthrough in reinforcement learning. It is widely studied for social good~\citep{nahum2012q,li2021claim}, and it has also been extensively used in RMAB problems~\citep{fu2019towards,avrachenkov2020Whittle,biswas2021learn} to estimate the expected Q-value, $Q^\ast(s,a,l)$, of taking action $a\in\{0,1\}$ after $l\in\{1,\dots,L\}$ time slots since last observation $s\in\{0,1\}$. The off-policy TD control algorithm is defined as
\begin{equation}\label{eq:update_Q}
\resizebox{\linewidth}{!}{$
\begin{aligned}
        Q^{t+1}&(s_t,a_t,l_t) \leftarrow Q^t(s_t,a_t,l_t)+\\
        &\alpha_t(s_t,a_t,l_t) \left[R_{t+1}+\gamma \underset{a}{\max}\left(Q^t(s_{t+1}, a,l_{t+1}) - Q^t(s_t,a_t,l_t)\right)\right]
\end{aligned}
$}    
\end{equation}
Where $\gamma$ is the discount rate, $\alpha_t(s_t,a_t,l_t)\in [0,1]$ is the learning rate parameter, i.e., a small $\alpha_t(s_t,a_t,l_t)$ will result in a slow learning process and no update when $\alpha_t(s_t,a_t,l_t)=0$. While a large $\alpha_t(s_t,a_t,l_t)$ may cause the estimated Q-value to rely heavily on the most recent return, when $\alpha_t(s_t,a_t,l_t)=1$, the Q-value will always be the most recent return.

We now describe how to use the Whittle index-based Q-Learning mechanism to solve the RMAB problem with fairness constraints.
We build on the work by \citet{biswas2021learn} for fully observable settings. In addition to considering fairness constraints, our model can be viewed as an extension to the partially observable setting.  Due to fairness constraints, $l$ can be a maximum of $L$ time steps.  Therefore, belief space is also limited. We are able to use the Q-Learning based approach to effectively compute the Whittle index value and this approach is summarized in Algorithm~\ref{al:alg3},

\begin{algorithm}[t]
\caption{Fair Whittle Thresholding based Q-Learning(FaWT-Q)}
\label{al:alg3}
\SetAlgoLined
\KwIn{parameter $\epsilon$ and $k$, and $\alpha_t(s_t,a_t,l_t)$, initial observed state set $\{s \}^N$, }
\For{each arm i in $1$ to $N$}{
    Initialize the $Q_i(s,a,l)\leftarrow 0$ for each state $s\in\{0,1\}$, and each action $a\in\{0,1\}$ and time length $l\in\{1,\dots,L\}$\; 
    For each $s\in\{0,1\}$ and $l\in \{1,\dots,L\}$\, initialize the Whittle index value set $\lambda_i(s,l)\leftarrow 0$\;
}
\For{t from $1$ to $T$}{
    \For{arm i in 1 to N}{
    \If{the fairness constraint is violated}{
                Add arm $i$ to the action set $\phi$\;
                $k=k-1$\;
            }
    }
    With prob $\epsilon$ add random $k$ arms to $\phi$ and with prob $1-\epsilon$ add arms with top $k$ $\lambda_i(s,l)$ value \;
    Activate the selected arms and receive rewards and observations\;
    \For{each arm i in $1$ to $N$}{
        Update the $Q_i^{t+1}(s,a,l)$ according to Eq.~\ref{eq:update_Q}\;
        \If{$i\in\phi$}{
            Set $l=1$ and update $s_i$ according to the received observation\;
        }
        \Else{
            Set $l=l+1$\;
        }
        Update the new Q-Learning based Whittle index by $\lambda_i^{t+1}(s,l)=Q_i(s,a=1,l)-Q_i(s,a=0,l)$
    }
}
\KwOut{Action set $\phi$}
\end{algorithm}



One typical form of $\alpha_t(s_t,a_t,l_t)$ could be $1/z(s_t, a_t,l_t)$, where $z(s_t, a_t,l_t)=\left(\sum_{u=0}^t \mathbb{I}\{s_u=s,a_u=a,l_u=l\}\right)+1$ for each initial observed state $s\in\{ 0,1\}$, action $a\in\{0,1 \}$ and time length since last activation $l\in\{1,\dots, L \}$ at the time slot $u$ from the beginning. With such mild form of $\alpha_t(s_t,a_t,l_t)$, we now are able to build the theoretical support for the Q-Learning based Whittle index approach.
\begin{theorem}\label{thm:q_learning}
Selecting the highest-ranking arms according to the $Q_i^{\ast}(s,a=1,l)-Q_i^{\ast}(s,a=0,l)$ till the budget constraint is met is equivalent to maximizing $\left\{ \sum_{i=1}^N Q_i^\ast(s,a,l)\right\}$ over all possible action set $\{0,1\}^N$ such that $\sum_{i=1}^N a_i=k$.
\end{theorem}
\begin{proof_sketch}
A proof based on work by~\citep{biswas2021learn} is given in Appendix. 
\end{proof_sketch}

\begin{theorem}\label{thm:Q_2}
{Stability and convergence}: The proposed approach converges to the optimal with probability $1$ under the following conditions:\\
\noindent {1.} The state space and action space are finite;\\
\noindent {2.} $\sum_{t=1}^{\infty}\alpha_t(s_t,a_t,l_t)=\infty \quad \sum_{t=1}^{\infty}\alpha_t^2 (\omega_i(t)) <\infty$
\end{theorem}
\begin{proof_sketch}
The key to the convergence is contingent on a particular sequence of episodes observed in the real process~\citep{watkins1992q}.
Detailed proof is given in Appendix. 
\end{proof_sketch}

\begin{figure*}
  \centering
    \includegraphics[width=0.92\textwidth]{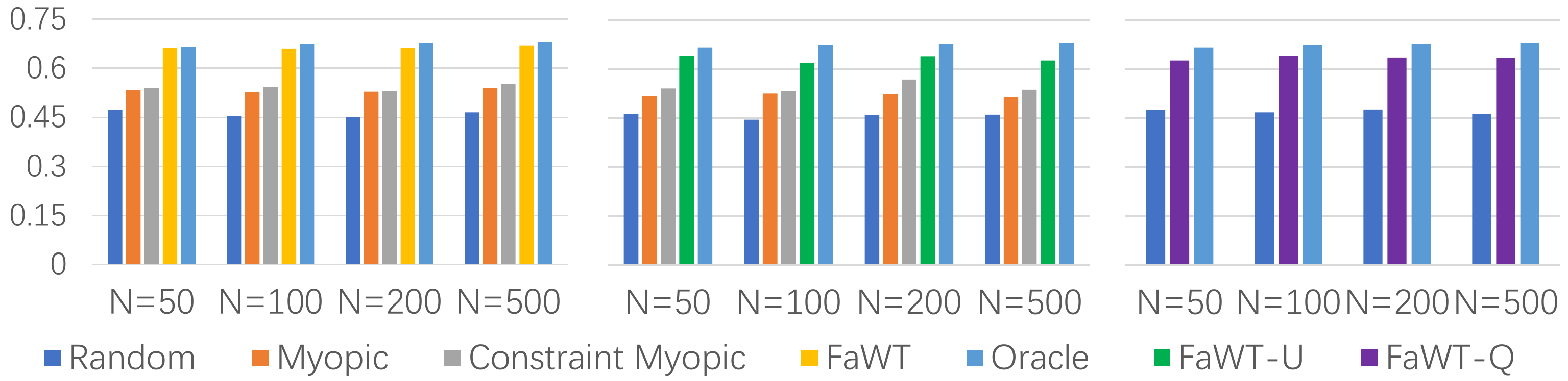}
  \caption{Comparison of performance of our approach and baseline approaches}
  \label{fig:performance}
\end{figure*}

\begin{figure*}[ht]
  \centering
    \includegraphics[width=0.92\textwidth]{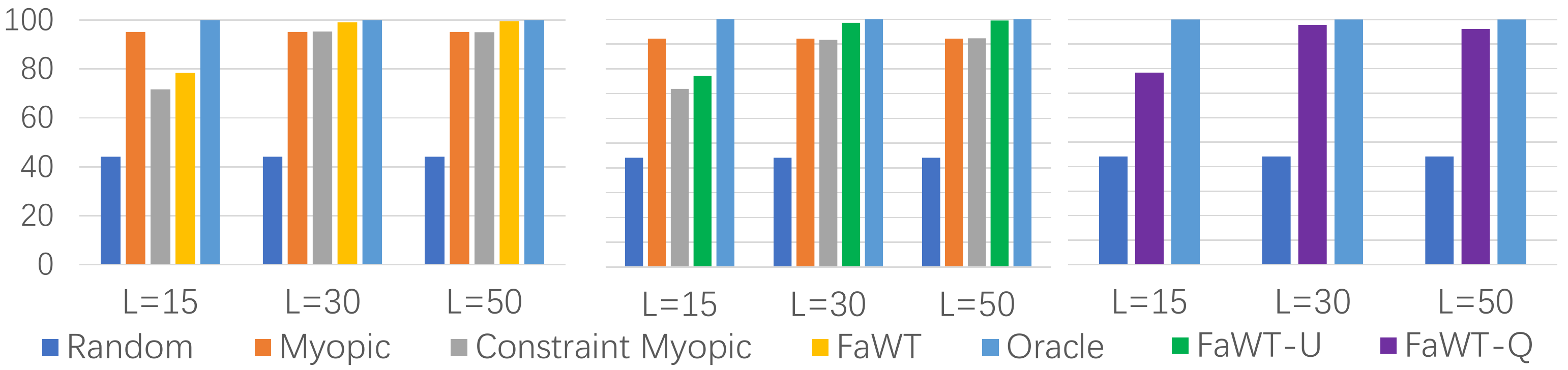}
  \caption{Intervention benefit ratio of our approach and baseline approaches without penalty for the violation of the fairness constraint. We set $N=100$, $k=10$, $T=1000$, $\eta=2$ and $L=\{ 15,30,50\}$.}
  \label{fig:benifit}
\end{figure*}

\section{Experiment}
To the best of our knowledge, we are the first to explore  fairness constraints in RMAB, hence the goal of the experiment section is to evaluate the performance of our approach in comparison to existing baselines:\\
\noindent \textbf{Random}: At each round, decision-maker randomly select $k$ arms to activate.\\
\noindent \textbf{Myopic}: Select $k$ arms that maximize the expected reward at the immediate next round. A myopic policy ignores the impact of present actions on future rewards and instead focuses entirely on predicted immediate returns. Formally, this could be described as choosing the $k$ arms with the largest gap $\Delta \omega_t = (\omega_{t+1}|a_t=1)-(\omega_{t+1}|a_t=0)$ at time $t$.\\
\noindent \textbf{Constraint Myopic:} It is the same as the Myopic when there is no conflict with fairness constraints, but if the fairness constraint is violated, it will choose the arm that satisfies the fairness constraint to play.\\
\noindent \textbf{Oracle}: Algorithm by~\citet{qian2016restless} under the assumption that the states of all arms are fully observable and the transition probabilities are known without considering fairness constraints.\\
To demonstrate the performance of our proposed methods, we test our algorithms on synthetic domains~\citep{mate2020collapsing} and provide numerical results averaged over 50 runs.

\paragraph{Average reward value with penalty:} In Figure~\ref{fig:performance}, we show the average reward $\bar{R}$ at each time step received by an arm over the time interval $T=1000$ for $N=50,100,200,500$ and $k=10\% \times N$ with the fairness constraint $L=20$, and $\eta=2$. We will receive a reward of $1$ if the state of an arm is $s=1$, and no reward otherwise.
We impose a small penalty of $-0.01$ if the fairness constraint of an arm is not satisfied. 
The graph on the left shows the performance of FaWT method when assuming the transition matrix is known. The middle graph is the average reward obtained using the FaWT-U approach when the transition model is not fully available. The right graph illustrates the result of FaWT-Q method when the transition model is unknown. As shown in the figure, our approaches consistently outperform the Random and Myopic baselines, and in addition to satisfying the fairness constraints, they have a near-optimal performance with a small difference gap when compared to the Oracle baseline. Note that the Myopic approach may fail in some cases(shown in~\cite{mate2020collapsing}), it performs worse than the Random approach.

\paragraph{No penalty for the violation of the fairness constraint:}

We also investigate the intervention benefit ratio defined as $\frac{\bar{R}_{\text{method}}-\bar{R}_{\text{No intervention}}}{\bar{R}_{\text{Oracle}}-\bar{R}_{\text{No intervention}}}\times 100\%$, where $\bar{R
}_{\text{No intervention}}$ denotes the average reward without any intervention involved. Here, we do not employ penalties when the fairness constraint is not satisfied, as we want to evaluate the benefit provided by interventions with our fair policy and policies of other approaches. 
We provide the intervention benefit ratio for different values of $L$ for all approaches in Figure~\ref{fig:benifit}. Again, the left graph shows the result of FaWT approach, the middle graph is the result of FaWT-U approach, and the right graph shows the result of FaWT-Q method. Our proposed approaches can achieve a better intervention benefit ratio compared with the baseline when $L$ is 30 and above. However, for L = 15, where there is a strict fairness constraint (i.e., $\frac{k\times L}{(\eta-1)\times N }$ is close to 1), it has a significant impact on solution quality. The performances of all our approaches improve when the fairness constraint's strength decreases ($L$ increases). Overall, our proposed methods can handle various levels of fairness constraint strength without sacrificing significantly on solution quality.


We also provide the additional experiment result that studies the influence of intervention level and fairness constraint's strength in the Appendix.

\section{Conclusion}
In this paper, we initiate the study of fairness constraints in Restless Multi-Arm Bandit problems. We define a fairness metric that encapsulates and generalizes existing fairness definitions employed for Multi-Arm Bandit problems. Contrary to expectations, we are able to provide minor modifications to the existing algorithm for RMAB problems in order to handle fairness. We provide theoretical results on how our methods provide the best way to handle fairness without sacrificing solution quality. This is demonstrated empirically as well on benchmark problems from the literature.

\begin{acknowledgements}

This research/project is supported by the National Research Foundation, Singapore under its AI Singapore Programme (AISG Award No: AISG2-RP-2020-017).
\end{acknowledgements}

\bibliography{li_493}

\clearpage

\appendix
\section{More information}
\subsection{Whittle Index Policy}
\label{app:Whittle}

We take advantage of the Fast Whittle Index Computation algorithm introduce in~\citet{mate2020collapsing}. They derived a closed form for computing the Whittle index for both average reward and discounted reward criterion, where the objective could also be written as $\bar{R}_{\lambda}^{\pi}=\mathbb{E} \sum_{\omega} f^{\pi}(\omega) R_{a^{\pi}}(\omega)$, where $f^{\pi}(\omega)$ is defined as the fraction of time spent in each belief state $\omega$ induced by policy $\pi$ and $f^{\pi}(\omega)\in [0,1]$. Their proposed Whittle index computation algorithm can achieve a 3-order-of-magnitude speedup compared to~\citet{qian2016restless}. In the two-states setting ($s\in \{0,1\}$), they use a tuple $(B_0^{\omega_{th}}, B_1^{\omega_{th}})$ to denote the belief threshold, where $\omega_{th}\in [0,1]$, and $B_0^{\omega_{th}},B_1^{\omega_{th}} \in 1,\dots,L$ are the index of the first belief state in each chain where it is optimal to act (i.e., the belief is less than or equal to $\omega_{th}$). The length is at most $L$ long due to our fairness constraints. This is defined as the forward threshold policy, and \citet{mate2020collapsing} used the Markov chain structure to derive the occupancy frequencies for each belief state $\omega_s(t)$, which is as follows,
\begin{equation}
  f^{(B_0^{\omega_{th}},B_1^{\omega_{th}})}(\omega_s(t)) =
    \begin{cases}
      a & \text{if $s=0$, $t\leq B_0$}\\
      b & \text{if $s=1$, $t\leq B_1$}\\
      0 & \text{otherwise}
    \end{cases}       
\end{equation}
\begin{equation}
\resizebox{\linewidth}{!}{$
    a=\left(\frac{B_1 \omega_0(B_0)}{1-\omega_1(B_1)} + B_0\right)^{-1},b=\left(\frac{B_1\omega_0(B_0)}{1-\omega_1(B_1)}+B_0\right)^{-1}\frac{\omega_0(B_0)}{1-\omega_1(B_1)}
    $}
\end{equation}
These occupancy frequencies do not depend on the subsidy $\lambda$. For the forward threshold policy $(B_0^{\omega_{th}}, B_1^{\omega_{th}})$, they use the $R_{\lambda}^{B_0^{\omega_{th}}, B_1^{\omega_{th}}}$ to denote the average reward, then can decompose the average reward into the contribution of the state reward and the subsidy $\lambda$
\begin{equation}
\resizebox{\linewidth}{!}{$
\begin{aligned}
            R_{\lambda}^{(B_0^{\omega_{th}}, B_1^{\omega_{th}})} &= \sum_{\omega \in \mathscr{B}} \omega f^{(B_0^{\omega_{th}},B_1^{\omega_{th}})}(\omega)\\
            &+\omega \left(1-f^{(B_0^{\omega_{th}},B_1^{\omega_{th}})}(\omega_1(B_1)) - f^{(B_0^{\omega_{th}},B_1^{\omega_{th}})}(\omega_0(B_0))\right)
\end{aligned}
$}
\end{equation}
Given the definition of the Whittle index $\lambda$, this could be interpreted to two corresponding threshold policies being equally optimal.
More specifically, for a belief state $\omega_0(B_0)$, the two adjacent threshold polices $\{ (B_0^{\omega_{th}},B_1^{\omega_{th}}),(B_0^{\omega_{th}}+1,B_1^{\omega_{th}}) \}$ would be optimal to be active and passive respectively. 
recall that the Whittle index is the smallest $\lambda$ for which the active and the passive actions are both optimal. Thus the subsidy which makes the average reward of those two adjacent polices equal in value must be the Whittle index for the belief state $\omega_0(B_0)$. Formally, this could be calculated through $R_{(\lambda}^{B_0^{\omega_{th}}, B_1^{\omega_{th}})}=R_{(\lambda}^{B_0^{\omega_{th}}, B_1^{\omega_{th}}+1)}$. Similarly, we can obtain the Whittle index for the belief state $\omega_1(B_1)$ through $R_{\lambda}^{(B_0^{\omega_{th}}, B_1^{\omega_{th}})}=R_{\lambda}^{(B_0^{\omega_{th}}, B_1^{\omega_{th}}+1)}$. These computations are repeated for every belief states to find the minimum subsidy value while $B_s^{\omega_{th}}\leq L$. The main idea of their approach is shown in Fig~\ref{fig:forward_reverse_policy}.

\begin{figure}[ht]
    \centering
    \includegraphics[width=0.9\linewidth]{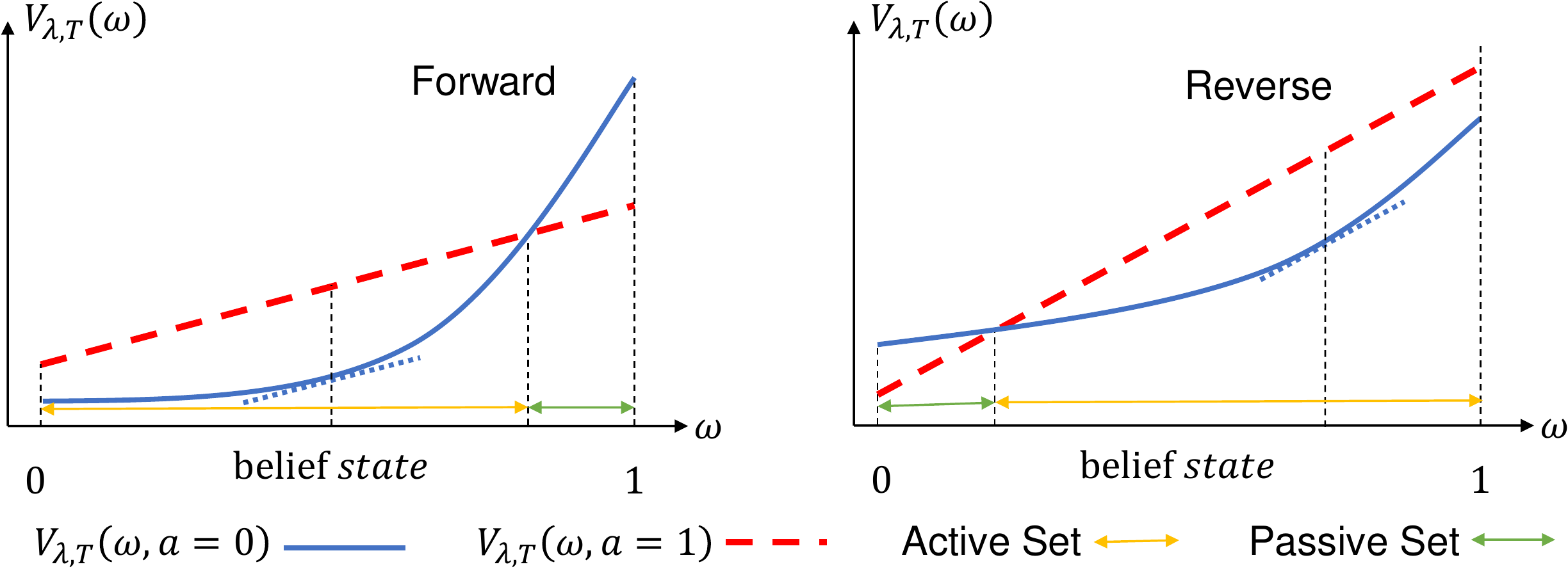}
    \caption{The forward and reverse policy}
    \label{fig:forward_reverse_policy}
\end{figure}

\subsection{Discussion of Fairness choice}
It is natural to ask what makes the proposed notion of fairness in this paper the right one? Our proposed fairness constraint is driven by the flaw of SOTA that a substantial number of arms are never selected, such starvation of intervention results in a huge demand for fairness requirement in the real-world. Furthermore, while our proposed algorithm can still be used, our notion of fairness can also be extended to fairness on the group/type of arms (i.e., check if the group fairness requirement is violated, and if so, select the arm with the highest index value in that group/type).
One of the most widely used fairness notion is to define a minimum rate that is required when allocating resources to users, and our fairness constraint can be viewed as a variant of this~\citep{chen2020fair,li2019combinatorial,patil2020achieving}. Another form of fairness constraint is to require that the algorithm never prefers a worse action over a better one based on the expected immediate reward~\citep{joseph2016fairness}. This can be seen as a variant of the Myopic algorithm in our baselines, which fails to satisfy our fairness constraint and performs worse compared to our method. Meanwhile, our Whittle index-based method can naturally satisfy such form of fairness constraint in the most time based on the long-term reward rather than the immediate reward. There might be many measures of fairness and it may be impossible to satisfy multiple types of fairness simultaneously (COMPAS Case Study). However, our proposed fairness constraint is one of most appropriate form in the real world. Namely, in the field of medical interventions, we can meet the requirement that everyone will receive medical treatment without sacrificing a significant overall performance, while SOTA/Myopic will only favor certain beneficiaries.

\section{Fully Observable Setting}
In order to explain the key idea, we initially consider a fully observable environment and then extend to a partially observable setting. Algorithm~\ref{al:alg1} provides the algorithm at each iteration for selecting the arms to activate (action set) in fully observable case (for both finite and infinite hoirzon cases).

\subsection{Infinite Horizon, Fully Observable}

We first provide the expression for $\lambda$ without the fairness constraints. 
We assume that $V_{\lambda, \infty}(s)$ denotes the value function which can be accrued from a single-armed bandit process with subsidy $\lambda$  over infinite horizon if the observed state is $s$.

Therefore, for the state 0, we have:\begin{equation}
\resizebox{\linewidth}{!}{$
    \displaystyle
\begin{aligned}
  V_{\lambda,\infty}(0) &= \max
    \begin{cases}
      \lambda +\beta (P_{0,1}^p V_{\lambda,\infty}(1)+P_{0,0}^p V_{\lambda,\infty}(0)) & \text{passive}\\
       \beta(P_{0,1}^a V_{\lambda,\infty}(1)+P_{0,0}^a V_{\lambda,\infty}(0)) & \text{active}
    \end{cases}\\
    &= \max
    \begin{cases}
      \lambda  +\beta (P_{0,1}^p (V_{\lambda,\infty}(1)- V_{\lambda,\infty}(0))+ V_{\lambda,\infty}(0))
    \\
       \beta(P_{0,1}^a (V_{\lambda,\infty}(1)-V_{\lambda,\infty}(0))+V_{\lambda,\infty}(0))
    \end{cases}
\end{aligned}
$}
\end{equation}
Similarly for state 1, we have
\begin{equation} \label{eq:mdp2}
\resizebox{\linewidth}{!}{$
    \displaystyle
  V_{\lambda,\infty}(1) = \max
    \begin{cases}
      \lambda +  1 +\beta (P_{1,1}^p (V_{\lambda,\infty}(1)- V_{\lambda,\infty}(0))+ V_{\lambda,\infty}(0)) & \text{passive}\\
       1 + \beta(P_{1,1}^a (V_{\lambda,\infty}(1)-V_{\lambda,\infty}(0))+V_{\lambda,\infty}(0)) & \text{active}
    \end{cases}     
    $}
\end{equation}
Recall that the definition of the Whittle index $W_T(s)$ of state $s$ is the smallest $\lambda$ s.t. it is optimal to make the arm passive in the current state.
Therefore, we have $\lambda = \beta \left((P_{0,1}^a - P_{0,1}^p)(V_{\lambda,\infty}(1)-V_{\lambda,\infty}(0))\right)$ for $s=0$ and, 
$\lambda = \beta \left((P_{1,1}^a - P_{1,1}^p)(V_{\lambda,\infty}(1)-V_{\lambda,\infty}(0))\right)$ for $s=1$. 

As a result, the Whittle index based approach would rank the index value of all $N$ arms and select top $k$ arms at each time step to activate. With fairness constraints, the change to the approach is minimal and intuitive. The optimal policy is still to choose the arms with the top "k" index values until a fairness constraint is violated for an arm. In that time step, we replace the last arm in top-k with the arm for which fairness constraint is violated. We show that this simple change works across the board for infinite and finite horizon, fully and partially observable settings.

\begin{theorem}\label{thm:FC_infinite_fully}
Algorithm~\ref{al:alg1} 
is optimal for RMAB with Fairness Constraints governed by parameters $\eta$ and $L$ under certain conditions.
\end{theorem}
\begin{proof_sketch}
This can be viewed as a special case of the partially observable setting. Please refer to the detailed proof for the partially observable case.
\end{proof_sketch}

\begin{figure}
    \centering
    \includegraphics[width=0.98\linewidth]{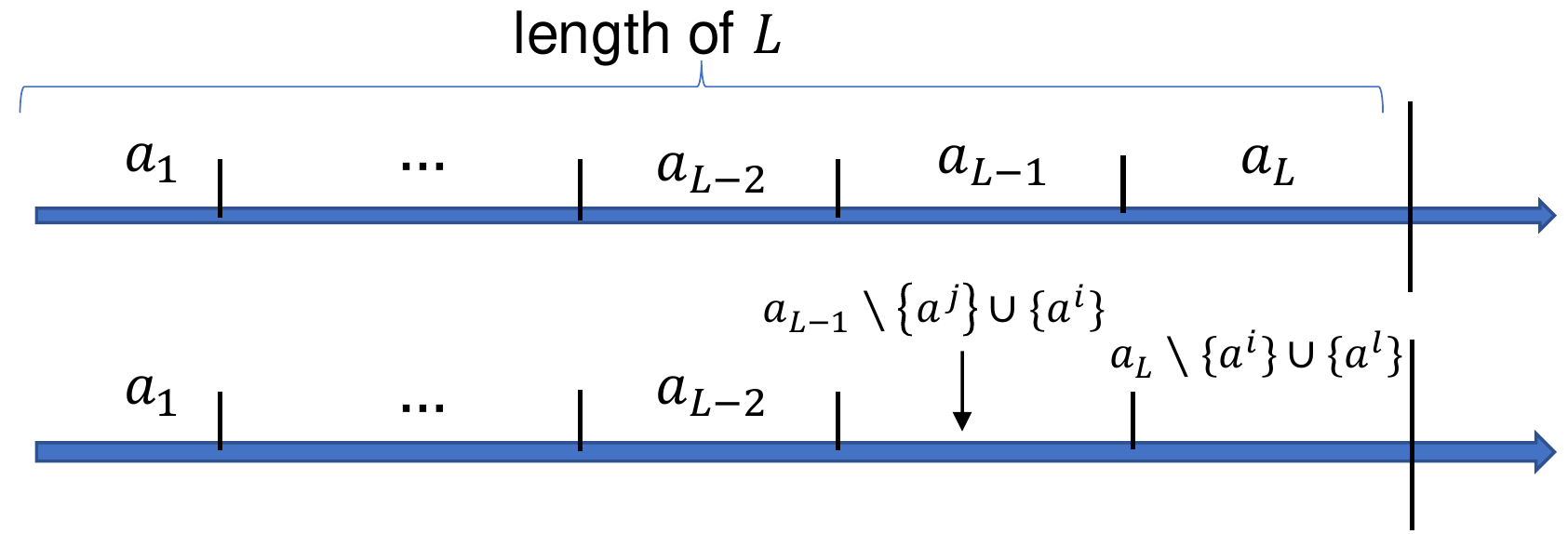}
    \caption{The action is $a_t$ at time step $t$. Then we move the action $a^i$ that satisfies the fairness constraint to one slot earlier to replace the $k-th$ action $a^j$, and add the action $a^l$ according to the index value at the end of the time interval.}
\end{figure}

\subsection{Finite Horizon, Fully Observable} 
Existing literature~\cite{glazebrook2006some, villar2016indexability} on RMAB deals with situations where the time horizon is infinite. In this part, we will demonstrate how the methodology discussed for infinite horizon setting to deal with fairness can also be applied to the finite horizon setting. There are two key challenges:  (i) computing the Whittle index under the finite horizon setting; (ii) to show that Whittle index value reduces as residual life time decreases;

It is costly to compute the index under the finite horizon setting ($O(|S|^kT)$ time and space complexity~\cite{hu2017asymptotically}; instead, we can take advantage of the fact that the index value will converge as $T\rightarrow \infty$ and $\lambda_0=0$. We can use a sigmoid curve to approximate the index value. One common example of a sigmoid function is the logistic function. This form is also used by~\citet{mate2021efficient}. Specifically, we let
\begin{equation}\label{eq:finite_Whittle_fully}
    {\displaystyle  W_T(s)={\frac {A}{1+e^{-k T}}}+C,}
\end{equation}
where ${\displaystyle A}$ and ${\displaystyle \frac{A}{2}+C}$ are the curve's bounds; ${\displaystyle k}$ is the logistic growth rate or steepness of the curve. We have $W_0(s)=0$ and $W_1(s)=\beta(s(P_{1,1}^a-P_{1,1}^p)+(1-s)(P_{0,1}^a-P_{0,1}^p))$, and let $W_{\infty}(s)=TW(s)$, where $TW(s)$ is the Whittle index value for state $s$ under infinite horizon. By solving these three constraints, we can get the three unknown parameters, $$C=-TW(s),A=2TW(s),$$
\begin{equation*}
\resizebox{\linewidth}{!}{$
    k=-\log (\frac{2TW(s)}{\beta(s(P_{1,1}^a-P_{1,1}^p)+(1-s)(P_{0,1}^a-P_{0,1}^p))+TW(s)}-1)
    $}
\end{equation*}

Algorithm~\ref{al:alg1} shows how to use $W_T(s)$ in considering fairness constraint under the finite horizon setting. As for optimality of Algorithm~\ref{al:alg1} in case of finite horizon, we have to show that value function and Whittle index decays over time in case of finite horizon, i.e.:
$$V_{\lambda_{T},T}(s)>V_{\lambda_{T-1},T-1}(s), \forall s\in \{0,1\}$$
$$\forall T>1: \lambda_T>\lambda_{T-1}>\lambda_0 = 0$$
This is to make the same argument as in case of infinite horizon, where we show that activating the fairness ensuring arm one step earlier results in lower value.

We show a lemma first, which can lead to the proof of the Whittle index decay.
\begin{lemma}\label{lem:inc}
Intuitively, we can have $V_{\lambda_{T},T}(s)>V_{\lambda_{T-1},T-1}(s)$, for $\forall s\in \{0,1\}$.
\end{lemma}

\begin{proof}\label{proof:fully_obs2}
For state $s\in \{0,1 \}$, we can always find a policy that ensures $V_{\lambda_{T},T}(s)>V_{\lambda_{T-1},T-1}(s)$. For example, we assume the optimal policy for the state $s$ with the residual time horizon $T-1$ is $\pi$, we can always find a policy $\pi^{\prime}$: keep the same policy for the first $T-1$ time slot for the same state $s$ with the residual time horizon $T$, and then pick the action for the last slot $T$ according to the observed state $s^{\prime}$. Because the reward is either $0$ or $1$, we could have 
\begin{equation}
\resizebox{\linewidth}{!}{$
V_{\lambda_{T},T;\pi^{\prime}}(s) = V_{\lambda_{T-1},T-1;\pi}(s) + \beta^{T} V_{\lambda_{1},1}(s^\prime)>V_{\lambda_{T-1},T-1;\pi}(s).
$}
\end{equation}
\end{proof}
Note this lemma is also suitable for the partially observable case.
Then we present the Theorem and corresponding proof.

\begin{theorem}\label{thm:decay_fully}
For a finite horizon $T$, the Whittle index $\lambda_T$ is the value that satisfies the equation $V_{\lambda_T,T}(s,a=0)= V_{\lambda_T,T}(s,a=1)$ for the observed state $s$. Assuming indexability holds, the Whittle index will decay as the value of horizon $T$ decreases: $\forall T>1: \lambda_T>\lambda_{T-1}>\lambda_0 = 0$.
\end{theorem}

\begin{proof_sketch}
Because the state $s$ is fully observable, We can easily calculate $\lambda_0$ and $\lambda_1$ by solving equation $V_{\lambda,T}^p(s)=V_{\lambda,T}^a(s)$. We then could derive $\lambda_T>\lambda_{T-1}$ by obtaining $\frac{\partial \lambda_T}{\partial T}>0$ for $s\in \{ 0,1\}$. Detailed proof in appendix.
\end{proof_sketch}
\begin{proof}
Consider the discount reward criterion with discount factor of $\beta\in [0,1]$ (where $\beta =1 $ corresponds to the average criterion). 
\begin{equation}
\resizebox{\linewidth}{!}{$
\begin{aligned}
        &s + \lambda +\beta (s P_{1,1}^p + (1-s)P_{0,1}^p)= s+\beta (s P_{1,1}^a+(1-s)P_{0,1}^a)\\
        \rightarrow &\lambda =\beta(s(P_{1,1}^a-P_{1,1}^p)+(1-s)(P_{0,1}^a-P_{0,1}^p)) 
\end{aligned}
$}
\end{equation}
Because $P_{1,1}^a-P_{1,1}^p>0$ and $P_{0,1}^a-P_{0,1}^p>0$, we have $\lambda>0$. Now we show $\lambda_T>\lambda_{T-1}$. Equivalently, this can be expressed as $\frac{\partial \lambda_T}{\partial T}>0$. Because the state is fully observable, we first get the close form of $\lambda_T$.
\begin{itemize}[leftmargin=*]
    \item{\verb|Case 1:|} The state $s=0$, 
    \begin{equation}
    \resizebox{1\linewidth}{!}{$
    \begin{aligned}
        &0 + \lambda_T +\beta (P_{0,0}^p V_{\lambda_{T-1},T-1}(0)+P_{0,1}^p V_{\lambda_{T-1},T-1}(1))\\
        = &0+\beta (P_{0,0}^a V_{\lambda_{T-1},T-1}(0)+P_{0,1}^a V_{\lambda_{T-1},T-1}(1))\\
        \rightarrow &\lambda_T =\beta(V_{\lambda_{T-1},T-1}(0)(P_{0,0}^a-P_{0,0}^p)+V_{\lambda_{T-1},T-1}(1)(P_{0,1}^a-P_{0,1}^p)).
\end{aligned}
$}
\end{equation}
Intuitively, we can have $V_{\lambda_{T},T}(0)>V_{\lambda_{T-1},T-1}(0)$ (see Lemma~\ref{lem:inc}), and  $V_{\lambda_{T},T}(1)>V_{\lambda_{T-1},T-1}(1)$, we obtain $\frac{\partial V_{\lambda_{T},T}(0)}{\partial T}>0$ and $\frac{\partial V_{\lambda_{T},T}(1)}{\partial T}>0$. Hence, we can get
\begin{equation}
\resizebox{\linewidth}{!}{$
\frac{\partial \lambda_T}{\partial T}
= \beta((P_{0,0}^a-P_{0,0}^p)\frac{\partial V_{\lambda_{T-1},T-1}(0)}{\partial T}+(P_{0,1}^a-P_{0,1}^p)\frac{\partial V_{\lambda_{T-1},T-1}(1)}{\partial T}) >0
$}
\end{equation}
    \item{\verb|Case 2:|} The state $s=1$,
\begin{equation}
\resizebox{\linewidth}{!}{$
    \begin{aligned}
        &1 + \lambda_T +\beta (P_{1,0}^p V_{\lambda_{T-1},T-1}(0)+P_{1,1}^p V_{\lambda_{T-1},T-1}(1))\\
        = &1+\beta (P_{1,0}^a V_{\lambda_{T-1},T-1}(0)+P_{1,1}^a V_{\lambda_{T-1},T-1}(1))\\
        \rightarrow &\lambda_T =\beta(V_{\lambda_{T-1},T-1}(0)(P_{1,0}^a-P_{1,0}^p)+V_{\lambda_{T-1},T-1}(1)(P_{1,1}^a-P_{1,1}^p)).
    \end{aligned}
    $}
\end{equation}
    Similarly, we have
\begin{equation}
\resizebox{\linewidth}{!}{$
\frac{\partial \lambda_T}{\partial T}
= \beta((P_{1,0}^a-P_{1,0}^p)\frac{\partial V_{\lambda_{T-1},T-1}(0)}{\partial T}+(P_{1,1}^a-P_{1,1}^p)\frac{\partial V_{\lambda_{T-1},T-1}(1)}{\partial T}) >0
$}
\end{equation}    
\end{itemize}
Thus $\forall T>1: \lambda_T>\lambda_{T-1}>\lambda_0 = 0$.
\end{proof}

The proof of optimality for Algorithm~\ref{al:alg1} in case of finite horizon is similar to the infinite horizon setting.

Most of the time, the states $\{s_1,\dots,s_N\}\in\{ 0,1\}^N$ are not observable until the action is taken. The decision-maker, on the other hand, can infer the current states based on the observation history by keeping track of the belief states $\omega$.

\section{Proofs for the Theorem~\ref{thm:FC_infinite} and Theorem~\ref{thm:FC_finite}}
\label{app:proof}

\subsection{Boundary Lemma}
\begin{lemma}\label{lem:boundary}
For the finite time horizon $T$,  $V_{\lambda_T,T}(\omega_1)-V_{\lambda_T,T}(\omega_2)$ is bounded, where we have 
\begin{equation}
\resizebox{\linewidth}{!}{$
\begin{aligned}
(\omega_1-\omega_2)\sum_{t=0}^{T-1}\beta^t(P_{1,1}^a-P_{0,1}^a)^t
\leq V_{\lambda_T,T}(\omega_1)-V_{\lambda_T,T}(\omega_2) \leq
(\omega_1-\omega_2)\sum_{t=0}^{T-1}\beta^t
\end{aligned}
$}
\end{equation}
\end{lemma}
\begin{proof}
We prove the lower bound by induction, and the upper bound can be proven similarly. 

When $T=1$, we start from the definition of the value function $V_{\lambda_T,T}(\omega)$ to have
\begin{itemize}[leftmargin=*]
    \item passive actions:
    \begin{equation}
    \resizebox{\linewidth}{!}{$
    \begin{aligned}
        V_{\lambda_1,1}(\omega_1,a=0)-V_{\lambda_1,1}(\omega_2,a=0) &= \lambda_1 +\omega_1 - \lambda_1 -\omega_2 \\&= \omega_1-\omega_2
    \end{aligned}
    $}
    \end{equation}
    \item active actions:
    \begin{equation}
    V_{\lambda_1,1}(\omega_1,a=1)-V_{\lambda_1,1}(\omega_2,a=1) =  \omega_1-\omega_2
    \end{equation}
\end{itemize}
We get $V_{\lambda_1,1}(\omega_1)-V_{\lambda_1,1}(\omega_2) = \omega_1 -\omega_2$. Now we assume
$V_{\lambda_T,T}(\omega_1)-V_{\lambda_T,T}(\omega_2) \geq(\omega_1-\omega_2)\sum_{t=0}^{T-1}\beta^t(P_{1,1}^a-P_{0,1}^a)^t$
hold for $\forall T>1$, then for time horizon $T+1$, we have
\begin{itemize}[leftmargin=*]
    \item passive actions:
    \begin{equation}
    \resizebox{\linewidth}{!}{$
    \begin{aligned}
        &V_{\lambda_{T+1},T+1}(\omega_1,a=0)-V_{\lambda_{T+1},T+1}(\omega_2,a=0) \\= 
        &\left(\lambda_1 +\omega_1 + \beta V_{\lambda_{T},T}(\omega_1(1))\right)- \left(\lambda_1 +\omega_2+\beta V_{\lambda_{T},T}(\omega_2(1))\right) \\
        =& \omega_1-\omega_2+\beta \left(V_{\lambda_{T},T}(\omega_1(1))-V_{\lambda_{T},T}(\omega_2(1))\right)\\
        \geq& \omega_1 - \omega_2 +\beta (\omega_1 -\omega_2) \sum_{t=0}^{T-1}\beta^t(P_{1,1}^a-P_{0,1}^a)^t\\
        \geq& \omega_1 - \omega_2 +(\omega_1 -\omega_2)\sum_{t=1}^{T}\beta^t(P_{1,1}^a-P_{0,1}^a)^t \text{ Line $\ast$}\\
        =&(\omega_1 -\omega_2)\sum_{t=0}^{T}\beta^t(P_{1,1}^a-P_{0,1}^a)^t
    \end{aligned}
    $}
    \end{equation}
    Line $\ast$ is because $0\leq P_{1,1}^a-P_{0,1}^a<1$.
    \item active actions:
    \begin{equation}
    \resizebox{\linewidth}{!}{$
    \begin{aligned}
        &V_{\lambda_{T+1},T+1}(\omega_1,a=1)-V_{\lambda_{T+1},T+1}(\omega_2,a=1) \\
        = &\left(  \omega_1 +\beta(\omega_1  V_{\lambda_{T},T}(P_{1,1}^a)+(1-\omega_1)V_{\lambda_{T},T}(P_{0,1}^a))   \right) \\ 
        &- \left(  \omega_2 +\beta(\omega_2 V_{\lambda_{T},T}(P_{1,1}^a)+(1-\omega_2)V_{\lambda_{T},T}(P_{0,1}^a))   \right) \\
        =&\omega_1-\omega_2 + \beta\left( (\omega_1-\omega_2)(V_{\lambda_{T},T}(P_{1,1}^a)-V_{\lambda_{T},T}(P_{0,1}^a))   \right)\\
        =& \omega_1-\omega_2+\beta(V_{\lambda_{T},T}(\omega_1(1))-V_{\lambda_{T},T}(\omega_2(1)))\\
        \geq& \omega_1 - \omega_2 +\beta(\omega_1-\omega_2)\sum_{t=0}^{T-1}\beta^t(P_{1,1}^a-P_{0,1}^a)^t\\
        \geq& \omega_1 - \omega_2 +(\omega_1-\omega_2)\sum_{t=1}^{T}\beta^t(P_{1,1}^a-P_{0,1}^a)^t \text{ Line $\ast$}\\
        =&(\omega_1-\omega_2)\sum_{t=0}^{T}\beta^t(P_{1,1}^a-P_{0,1}^a)^t
    \end{aligned}
    $}
    \end{equation}
\end{itemize}
Line $\ast$ is because $0\leq P_{1,1}^a-P_{0,1}^a<1$. Thus we have $V_{\lambda_T,T}(\omega_1)-V_{\lambda_T,T}(\omega_2)\geq 
(\omega_1-\omega_2)\sum_{t=0}^{T-1}\beta^t(P_{1,1}^a-P_{0,1}^a)^t$. Similarly, we can prove the lower bound.
\end{proof}

\begin{lemma}\label{lem:infinite_boundary}
For the infinite residual time horizon $T\rightarrow \infty$, $V_{\lambda_T,T}(\omega_1)-V_{\lambda_T,T}(\omega_2)$is bounded. Specifically, we have
\begin{equation}
\resizebox{\linewidth}{!}{$
\begin{aligned}
\frac{\omega_1 -\omega_2}{1-\beta(P_{1,1}^a-P_{0,1}^a)}
\leq V_{\lambda_T,T}(\omega_1)-V_{\lambda_T,T}(\omega_2) \leq
\frac{\omega_1-\omega_2}{1-\beta}
\end{aligned}
$}
\end{equation}
\end{lemma}
\begin{proof}
This can be viewed as a special case of the finite residual time horizon setting where $T\rightarrow \infty$.
Thus we can easily derive the lower and upper bound according to the formula for the geometric series:
\begin{equation*}
    \underset{T\rightarrow \infty}{\lim}(\omega_1-\omega_2)\sum_{t=0}^{T-1}\beta^t(P_{1,1}^a-P_{0,1}^a)^t = \frac{\omega_1 -\omega_2}{1-\beta(P_{1,1}^a-P_{0,1}^a)}
\end{equation*}
and
\begin{equation*}
    \underset{T\rightarrow \infty}{\lim}(\omega_1-\omega_2)\sum_{t=0}^{T-1}\beta^t = \frac{\omega_1-\omega_2}{1-\beta}
\end{equation*}
\end{proof}

Consider the single-armed bandit process with subsidy $\lambda$ under the infinite time horizon $T\rightarrow \infty$, we have:
\begin{equation}
\resizebox{\linewidth}{!}{$
    \displaystyle
\begin{aligned}
  V_{\lambda,\infty}(\omega) &= \max
    \begin{cases}
      \lambda+\omega +\beta  V_{\lambda,\infty}(\tau^1(\omega)) & \text{passive}\\
       \omega+\beta\left(\omega V_{\lambda,\infty}(P_{1,1}^a)+(1-\omega) V_{\lambda,\infty}(P_{0,1}^a)\right) & \text{active}
    \end{cases}
\end{aligned}
$}
\end{equation}
and we can get 
\begin{equation}\label{eq:infinite_partial}
\resizebox{\linewidth}{!}{$
    \displaystyle
    \begin{aligned}
    \frac{\partial V_{\lambda,\infty}(\omega)}{\partial \omega} &=
    \begin{cases}
       1 +\beta  \frac{ \partial V_{\lambda,\infty}(\tau^1(\omega))}{\partial \tau^1(\omega)} \frac{\partial \tau^1(\omega)}{\partial \omega} & \text{passive}\\
       1+\beta\left( V_{\lambda,\infty}(P_{1,1}^a)- V_{\lambda,\infty}(P_{0,1}^a)\right) & \text{active}
    \end{cases}
\end{aligned}
$}
\end{equation}

Similarly, for the finite residual time horizon $T$ we have:
\begin{equation}\label{eq:finite_partial}
\resizebox{\linewidth}{!}{$
    \displaystyle
    \begin{aligned}
    \frac{\partial V_{\lambda,T}(\omega)}{\partial \omega} &=
    \begin{cases}
       1 +\beta  \frac{ \partial V_{\lambda,T-1}(\tau^1(\omega))}{\partial \tau^1(\omega)} \frac{\partial \tau^1(\omega)}{\partial \omega} & \text{passive}\\
       1+\beta\left( V_{\lambda,T-1}(P_{1,1}^a)- V_{\lambda,T-1}(P_{0,1}^a)\right) & \text{active}
    \end{cases}
\end{aligned}
$}
\end{equation}
Note that for any belief state $\omega$, $\tau^1(\omega)$ is the $1$-step belief state update of $\omega$ when the passive arm is unobserved for another $1$ consecutive slot. According to the Eq.~\ref{eq:1}, we have $\tau^1(\omega) = \omega P_{1,1}^p + (1-\omega)P_{0,1}^p$, thus 
\begin{equation}\label{eq:partial_omega_1}
    0< \frac{\partial \tau^1(\omega)}{\partial \omega} = (P_{1,1}^p-P_{0,1}^p) <1
\end{equation}

\begin{lemma}\label{lem:partial_value_bound}
For the finite residual time horizon $T$, we have $\frac{\partial V_{\lambda_T,T}(\omega)}{\partial \omega} \geq \min \{1 +\beta(P_{1,1}^p-P_{0,1}^p) \sum_{t=0}^{T-2}[\beta^t(P_{1,1}^a-P_{0,1}^a)^t], 1+\beta (P_{1,1}^a-P_{0,1}^a) \sum_{t=0}^{T-2}[\beta^t(P_{1,1}^a-P_{0,1}^a)^t]   \}$
\end{lemma}
\begin{proof}
According to Eq.~\ref{eq:finite_partial}, for the passive action, we have:
\begin{equation}\label{eq:finite_partial_lower_1}
\resizebox{\linewidth}{!}{$
    \begin{aligned}
            &1 +\beta  \frac{ \partial V_{\lambda,T-1}(\tau^1(\omega))}{\partial \tau^1(\omega)} \frac{\partial \tau^1(\omega)}{\partial \omega}\\
            =&1 +\beta \underset{\delta\rightarrow 0}{\lim}\frac{V_{\lambda,T-1}(\tau^1(\omega)+\delta)-V_{\lambda,T-1}(\tau^1(\omega))}{\delta} (P_{1,1}^p-P_{0,1}^p)\\
    \end{aligned}
$}
\end{equation}
According to Lemma~\ref{lem:boundary}, let $\omega_1 = \tau^1(\omega)+\delta$ and $\omega_2=\tau^1(\omega)$, then we have $V_{\lambda,T-1}(\tau^1(\omega)+\delta)-V_{\lambda,T-1}(\tau^1(\omega))\geq (\omega_1-\omega_2)\sum_{t=0}^{T-1}\beta^t(P_{1,1}^a-P_{0,1}^a)^t = \delta \sum_{t=0}^{T-1}\beta^t(P_{1,1}^a-P_{0,1}^a)^t$. Thus Eq.~\ref{eq:finite_partial_lower_1} becomes:
\begin{equation}\label{eq:finite_partial_lower_2}
    \begin{aligned}
        &1 +\beta  \frac{ \partial V_{\lambda,T-1}(\tau^1(\omega))}{\partial \tau^1(\omega)} \frac{\partial \tau^1(\omega)}{\partial \omega}\\
        \geq& 1 +\beta(P_{1,1}^p-P_{0,1}^p) \sum_{t=0}^{T-2}[\beta^t(P_{1,1}^a-P_{0,1}^a)^t]\\
    \end{aligned}
\end{equation}
Similarly, for the active action, according to lower bound in Lemma~\ref{lem:boundary}, we have:
\begin{equation}\label{eq:finite_partial_lower_3}
    \begin{aligned}
            &1+\beta\left( V_{\lambda,T-1}(P_{1,1}^a)- V_{\lambda,T-1}(P_{0,1}^a)\right)\\
            \geq & 1+\beta (P_{1,1}^a-P_{0,1}^a) \sum_{t=0}^{T-2}[\beta^t(P_{1,1}^a-P_{0,1}^a)^t]\\
    \end{aligned}
\end{equation}
Therefore, we have $\frac{\partial V_{\lambda_T,T}(\omega)}{\partial \omega} \geq \min \{ (P_{1,1}^p-P_{0,1}^p), (P_{1,1}^a-P_{0,1}^a)\}\cdot \beta \cdot \sum_{t=0}^{T-2}[\beta^t(P_{1,1}^a-P_{0,1}^a)^t]+1$

\end{proof}

\begin{lemma}\label{lem:partial_value_bound_infinite}
For the infinite residual time horizon $T\rightarrow \infty$, we have $\frac{\partial V_{\lambda_T,T}(\omega)}{\partial \omega} \geq \min \{1 +\frac{\beta(P_{1,1}^p-P_{0,1}^p)}{ 1-(\beta(P_{1,1}^a-P_{0,1})}, \frac{1}{1-\beta (P_{1,1}^a-P_{0,1}^a)} \}$
\end{lemma}
\begin{proof}
The proof is similar to the proof for Lemma~\ref{lem:partial_value_bound} of the finite setting. We can get the result with assuming $T\rightarrow \infty$.
\end{proof}

\begin{lemma}\label{lem:partial_value_bound_upper_finite}
For the finite residual time horizon $T$, we have $\frac{\partial V_{\lambda_T,T}(\omega)}{\partial \omega} \leq \min \{1 +(P_{1,1}^p-P_{0,1}^p) \sum_{t=1}^{T-1}\beta^t, 1+ (P_{1,1}^a-P_{0,1}^a) \sum_{t=1}^{T-1}\beta^t   \}$
\end{lemma}
\begin{proof}
The proof is similar to the proof of Lemma~\ref{lem:partial_value_bound}. According to Eq.~\ref{eq:finite_partial}, we have:
\begin{itemize}[leftmargin=*]
    \item passive actions:
    \begin{equation}\label{eq:finite_partial_upper_1}
    \resizebox{\linewidth}{!}{$
    \begin{aligned}
            &1 +\beta  \frac{ \partial V_{\lambda,T-1}(\tau^1(\omega))}{\partial \tau^1(\omega)} \frac{\partial \tau^1(\omega)}{\partial \omega}\\
            =&1 +\beta \underset{\delta\rightarrow 0}{\lim}\frac{V_{\lambda,T-1}(\tau^1(\omega)+\delta)-V_{\lambda,T-1}(\tau^1(\omega))}{\delta} (P_{1,1}^p-P_{0,1}^p)\\
    \end{aligned}
$}
\end{equation}
According to Lemma~\ref{lem:boundary}, let $\omega_1 = \tau^1(\omega)+\delta$ and $\omega_2=\tau^1(\omega)$, then we have $V_{\lambda,T-1}(\tau^1(\omega)+\delta)-V_{\lambda,T-1}(\tau^1(\omega))\leq (\omega_1-\omega_2)\sum_{t=0}^{T-1}\beta^t = \delta \sum_{t=0}^{T-1}\beta^t$. Thus Eq.~\ref{eq:finite_partial_upper_1} becomes:
\begin{equation}\label{eq:finite_partial_upper_2}
    \begin{aligned}
        &1 +\beta  \frac{ \partial V_{\lambda,T-1}(\tau^1(\omega))}{\partial \tau^1(\omega)} \frac{\partial \tau^1(\omega)}{\partial \omega}\\
        \leq& 1 +\beta(P_{1,1}^p-P_{0,1}^p) \sum_{t=0}^{T-2}[\beta^t]\\
    \end{aligned}
\end{equation}
    \item active actions, similarly, according to upper bound in Lemma~\ref{lem:boundary}, we have:
\begin{equation}\label{eq:finite_partial_upper_3}
    \begin{aligned}
            &1+\beta\left( V_{\lambda,T-1}(P_{1,1}^a)- V_{\lambda,T-1}(P_{0,1}^a)\right)\\
            \leq & 1+\beta (P_{1,1}^a-P_{0,1}^a) \sum_{t=0}^{T-2}[\beta^t]\\
    \end{aligned}
\end{equation}
\end{itemize}
Therefore, we have $\frac{\partial V_{\lambda_T,T}(\omega)}{\partial \omega} \leq \min \{1 +(P_{1,1}^p-P_{0,1}^p) \sum_{t=1}^{T-1}\beta^t, 1+ (P_{1,1}^a-P_{0,1}^a) \sum_{t=1}^{T-1}\beta^t   \}$

\end{proof}

\begin{lemma}\label{lem:partial_value_bound_upper_infinite}
For the infinite residual time horizon $T\rightarrow \infty$, we have $\frac{\partial V_{\lambda_T,T}(\omega)}{\partial \omega} \leq \min \{ 1+\frac{ \beta (P_{1,1}^p-P_{0,1}^p)}{1-\beta}, 1+\frac{ \beta (P_{1,1}^a-P_{0,1}^a)}{1- \beta}   \}$
\end{lemma}
\begin{proof}
The proof is similar to the proof for Lemma~\ref{lem:partial_value_bound_upper_finite} of the finite setting. We can get the result with assuming $T\rightarrow \infty$.
\end{proof}

\subsection{Condition for the optimality of Algorithm~\ref{al:alg1} under finite/infinite horizon}
We now give the proof for the Theorem~\ref{thm:FC_infinite} and Theorem~\ref{thm:FC_finite}.

According to the Eq.~\ref{eq:1}, we can compute the belief gap between $\omega$ and 1-time step belief update $\tau^1(\omega)$:
\begin{equation}
    \Delta \omega = \tau^{1}(\omega) - \omega = (P_{1,1}^p - P_{0,1}^p-1)\omega +P_{0,1}^p
\end{equation}

Remark that we could get a strict condition that depends only on arm A. This is because change from policy $\pi^\ast$ to $\pi$ will only lead to a decrease in the value function for other arms as the optimal actions determined by the Whittle index algorithm will be influenced, henceforth the value will be decreased. Consider the single-arm A, as we discussed earlier, the belief state update process is either monotonically increasing or monotonically decreasing.



\begin{figure}[ht]
    \centering
    \includegraphics[width=0.98\linewidth]{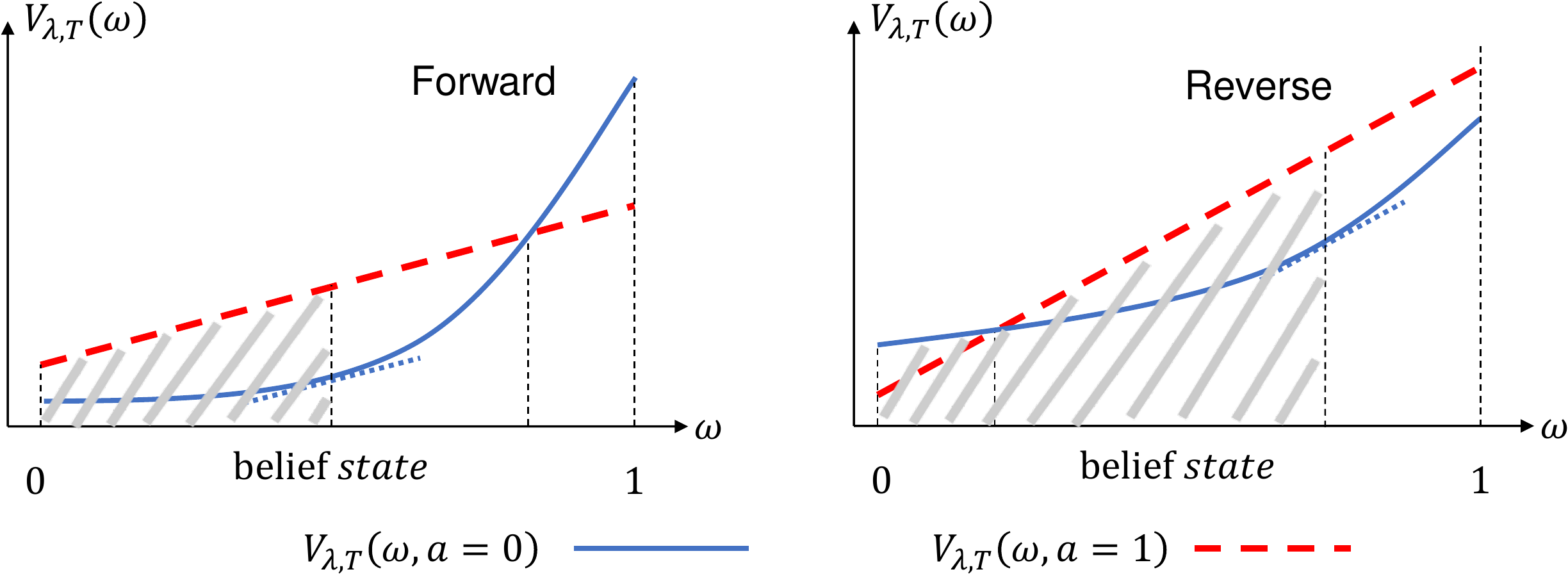}
    \caption{The forward and reverse policy}
    \label{fig:forward_reverse}
\end{figure}
\paragraph{Case 1}: The belief state monotonically increases as the time passed.
Formally, this can be expressed as $\frac{\partial \omega_{t}}{\partial t}>0$, or $\Delta \omega >0$.
We now derive the condition for the optimality of our algorithm for the case 1 under finite time horizon $T$. Consider first (any) period of length $L$, an arm $i$ has not been activated for the past $(L-1)$ time slots. Thus it needs to be pulled at time step $t=L$ according to our algorithm. 
Assume that the residual time horizon is $h$ at time step $t=L$, where we have $h+L=T$.
We move the active action at time step $t=L$ to one slot earlier, then at time step $t=L-1$, the residual time horizon is $h+1$, and assume that belief state is $\omega$ at time step $t=L-1$. We here discuss the finite horizon case because the infinite horizon could be viewed as a special case of the finite horizon setting as $T\rightarrow \infty (h\rightarrow \infty)$.

Because the belief state will increase as the time passed, thus we define the value gap $\Delta V_{\lambda,h}(\omega) =  V_{\lambda,h}(\omega,a=1)-V_{\lambda,h}(\omega,a=0)$ will move from left to the right as the residual time horizon decrease. For the single-arm process, if we can show that the gap difference $\Delta V_{\lambda,h}(\omega)$ increases from left to the right (i.e.,$\Delta V_{\lambda,h}(\omega)$ increases as belief state increases), then this implies that moving the active action that ensuring the fairness at time step $t=L$ to one step earlier (i.e., from right to left) will result in a smaller gap $\Delta V_{\lambda,h}(\omega)$. Thus it is optimal to keep the active action at the end of the period to ensure the fairness constraint. This requires that 
\begin{equation}
    \frac{\partial V_{\lambda,h}(\omega,a=0)}{\partial \omega}\leq \frac{\partial V_{\lambda,h}(\omega,a=1)}{\partial \omega}
\end{equation}

According to the expression for $\lambda$, we have
\begin{equation}\label{eq:deviation_finite}
\resizebox{\linewidth}{!}{$
    \displaystyle
    \begin{aligned}
    \frac{\partial V_{\lambda,h}(\omega)}{\partial \omega} &=
    \begin{cases}
       1 +\beta  \frac{ \partial V_{\lambda,h-1}(\tau^1(\omega))}{\partial \tau^1(\omega)} \frac{\partial \tau^1(\omega)}{\partial \omega} & \text{passive}\\
       1+\beta\left( V_{\lambda,h-1}(P_{1,1}^a)- V_{\lambda,h-1}(P_{0,1}^a)\right) & \text{active}
    \end{cases}
\end{aligned}
$}
\end{equation}

As shown in the gray area of the left Fig.~\ref{fig:forward_reverse}, at time step $t=L-1$, we derive the technical condition for the optimality of our algorithm in the gray area under the infinite residual time horizon:

\begin{equation}
\resizebox{1\linewidth}{!}{$
    \begin{aligned}
        &(P_{1,1}^p-P_{0,1}^p)\left(1+\frac{\beta \Delta_3}{1-\beta} \right){\left(1-\beta(P_{1,1}^a-P_{0,1}^a)\right)} \leq (P_{1,1}^a - P_{0,1}^a)\\
        \rightarrow& (P_{1,1}^p-P_{0,1}^p)\left(1+\frac{\beta \Delta_3}{1-\beta} \right) \leq \frac{P_{1,1}^a - P_{0,1}^a}{1-\beta(P_{1,1}^a-P_{0,1}^a)} \text{ Line 1}\\
        \rightarrow& (P_{1,1}^p-P_{0,1}^p)\left(1+\frac{\beta \Delta_3}{1-\beta} \right) \leq V_{\lambda,h-1}(P_{1,1}^a)- V_{\lambda,h-1}(P_{0,1}^a) \text{ Line 2}\\
        \rightarrow&  (P_{1,1}^p-P_{0,1}^p) \frac{ \partial V_{\lambda,h-1}(\tau^1(\omega))}{\partial \tau^1(\omega)} \leq V_{\lambda,h-1}(P_{1,1}^a)- V_{\lambda,h-1}(P_{0,1}^a) \text{ Line 3}\\      
        \rightarrow&  \frac{\partial \tau^1(\omega)}{\partial \omega} \frac{ \partial V_{\lambda,h-1}(\tau^1(\omega))}{\partial \tau^1(\omega)} \leq \left( V_{\lambda,h-1}(P_{1,1}^a)- V_{\lambda,h-1}(P_{0,1}^a)\right) \text{ Line 4}\\
        \rightarrow&  1 +\beta  \frac{ \partial V_{\lambda,h-1}(\tau^1(\omega))}{\partial \tau^1(\omega)} \frac{\partial \tau^1(\omega)}{\partial \omega}\leq 1+\beta\left( V_{\lambda,h-1}(P_{1,1}^a)- V_{\lambda,h-1}(P_{0,1}^a)\right) \text{ Line 5}\\
        \rightarrow &\frac{\partial V_{\lambda,h}(\omega,a=0)}{\partial \omega}\leq \frac{\partial V_{\lambda,h}(\omega,a=1)}{\partial \omega} \text{ Line 6}
    \end{aligned}
    $}
\end{equation}

Line 1 is obtained via  mathematical transformation. Line 2 is obtained from the lower bound in Lemma~\ref{lem:boundary}. Line 3 is obtained from the Lemma~\ref{lem:partial_value_bound_upper_infinite} when assuming $h\rightarrow \infty$. Line 4 is obtained from Eq.~\ref{eq:partial_omega_1}. Line 5 is obtained from the mathematical transformation. Line 6 is obtained from the Eq.~\ref{eq:deviation_finite}. And $\Delta_3 = \min \{(P_{1,1}^p-P_{0,1}^p),(P_{1,1}^a-P_{0,1}^a)\}$

Similarly, we can derive the technical condition for the finite residual time horizon, which is 

\begin{equation}
\resizebox{1\linewidth}{!}{$
    (P_{1,1}^p-P_{0,1}^p)\left(\Delta_4\beta\sum_{t=0}^{h-2}[\beta^t]+1\right) \leq (P_{1,1}^a - P_{0,1}^a)\sum_{t=0}^{h-2}[\beta^t(P_{1,1}^a-P_{0,1}^a)^t]
    $}
\end{equation}
where $\Delta_4= \min \{ (P_{1,1}^p-P_{0,1}^p), (P_{1,1}^a-P_{0,1}^a)\}$.

\begin{figure}[ht]
    \centering
    \includegraphics[width=0.9\linewidth]{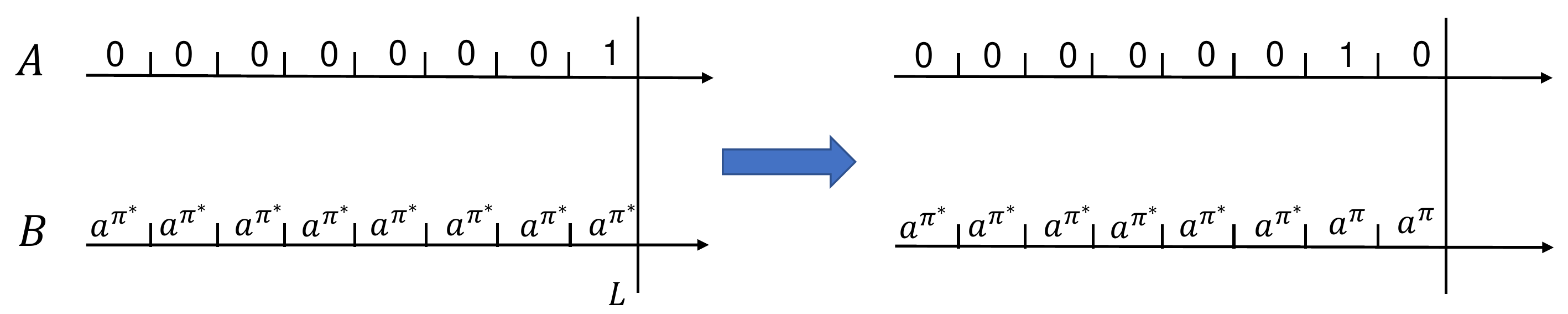}
    \caption{Proof of Theorem~\ref{thm:FC_infinite} and Theorem~\ref{thm:FC_finite}}
    \label{fig:fully_ob_fairness}
\end{figure}

\paragraph{Case 2}: The belief state monotonically decreases as the time passed. Formally, this can be expressed as $\frac{\partial \omega_{t}}{\partial t}<0$, or $\Delta \omega <0$.
Similarly, we can derive the condition for the optimality of our algorithm for the case 2 under finite time horizon $T$. 
Because the belief state will decrease as the time passed, thus we define the value gap $\Delta V_{\lambda,h}(\omega) =  V_{\lambda,h}(\omega,a=1)-V_{\lambda,h}(\omega,a=0)$ will move from right to the left as the residual time horizon decrease. For the single-arm process, if we can show that the gap difference $\Delta V_{\lambda,h}(\omega)$ decreases from right to the left (i.e., as belief state decrease, $\Delta V_{\lambda,h}(\omega)$ decreases), then this implies that moving the active action that ensuring the fairness at time step $t=L$ to one step earlier will result in a larger gap $\Delta V_{\lambda,h}(\omega)$. Thus it is optimal to keep the active action at the end of the period, i.e., at time step $t=L$. this requires that 
\begin{equation}
    \frac{\partial V_{\lambda,h}(\omega,a=0)}{\partial \omega}\geq \frac{\partial V_{\lambda,h}(\omega,a=1)}{\partial \omega}
\end{equation}

As shown in the gray area of the left Fig.~\ref{fig:forward_reverse}, at time step $t=L-1$, we derive the technical condition for the optimality of our algorithm in the gray area under the infinite residual time horizon:

\begin{equation}
\resizebox{1\linewidth}{!}{$
    \begin{aligned}
        &(P_{1,1}^p-P_{0,1}^p)(1-\beta)\Delta_1 \geq (P_{1,1}^a - P_{0,1}^a)\left(1-\beta(P_{1,1}^a - P_{0,1}^a)\right)\\
        \rightarrow& \frac{(P_{1,1}^p-P_{0,1}^p)\Delta_1}{1-\beta(P_{1,1}^a - P_{0,1}^a)} \geq \frac{P_{1,1}^a - P_{0,1}^a}{1-\beta} \text{ Line 1}\\
        \rightarrow& (P_{1,1}^p-P_{0,1}^p)\Delta_1\underset{h\rightarrow \infty}{\lim}\sum_{t=0}^{h-2}\beta^t(P_{1,1}^a-P_{0,1}^a)^t \geq \frac{P_{1,1}^a - P_{0,1}^a}{1-\beta} \text{ Line 2}\\
        \rightarrow& (P_{1,1}^p-P_{0,1}^p)\Delta_1\underset{h\rightarrow \infty}{\lim}\sum_{t=0}^{h-2}\beta^t(P_{1,1}^a-P_{0,1}^a)^t  \geq V_{\lambda,h-1}(P_{1,1}^a)- V_{\lambda,h-1}(P_{0,1}^a) \text{ Line 3}\\
        \rightarrow&  (P_{1,1}^p-P_{0,1}^p) \frac{ \partial V_{\lambda,h-1}(\tau^1(\omega))}{\partial \tau^1(\omega)} \geq V_{\lambda,h-1}(P_{1,1}^a)- V_{\lambda,h-1}(P_{0,1}^a) \text{ Line 4}\\      
        \rightarrow&  \frac{\partial \tau^1(\omega)}{\partial \omega} \frac{ \partial V_{\lambda,h-1}(\tau^1(\omega))}{\partial \tau^1(\omega)} \geq \left( V_{\lambda,h-1}(P_{1,1}^a)- V_{\lambda,h-1}(P_{0,1}^a)\right) \text{ Line 5}\\
        \rightarrow&  1 +\beta  \frac{ \partial V_{\lambda,h-1}(\tau^1(\omega))}{\partial \tau^1(\omega)} \frac{\partial \tau^1(\omega)}{\partial \omega}\geq 1+\beta\left( V_{\lambda,h-1}(P_{1,1}^a)- V_{\lambda,h-1}(P_{0,1}^a)\right) \text{ Line 6}\\
        \rightarrow &\frac{\partial V_{\lambda,h}(\omega,a=0)}{\partial \omega}\geq \frac{\partial V_{\lambda,h}(\omega,a=1)}{\partial \omega} \text{ Line 7}
    \end{aligned}
    $}
\end{equation}

Line 1 is obtained via  mathematical transformation. Line 2 is obtained from the formula for the geometric series as $\beta(P_{1,1}^a-P_{0,1}^a)<1$. Line 3 is obtained from the upper bound in Lemma~\ref{lem:boundary}. Line 4 is obtained from the Lemma~\ref{lem:partial_value_bound} when assuming $h\rightarrow \infty$. Line 5 is obtained from Eq.~\ref{eq:partial_omega_1}. Line 6 is obtained from the mathematical transformation. Line 7 is obtained from the Eq.~\ref{eq:deviation_finite}.

Similarly, we can derive the technical condition for the finite residual time horizon, which is 
\begin{equation}
\resizebox{1\linewidth}{!}{$
    (P_{1,1}^p-P_{0,1}^p)\left(\Delta_2\beta\sum_{t=0}^{h-2}[\beta^t(P_{1,1}^a-P_{0,1}^a)^t]+1\right) \geq (P_{1,1}^a - P_{0,1}^a)\sum_{t=0}^{h-2}\beta^t
    $}
\end{equation}
where $\Delta_2= \min \{ (P_{1,1}^p-P_{0,1}^p), (P_{1,1}^a-P_{0,1}^a)\}$.

When the belief state is in the white area of the passive set. Then we need to consider arm A and other arms in the active set. We give detailed discussion in Appendix~\ref{ap_subsec:general_condition}.

\section{Other Proofs}
\subsection{Derivation of Equation}
The belief state can be calculated in closed form with the given transition probabilities. Let $\tau^u_i(\omega_{i,s}(t))=\omega_{i,s}(t+u)$ denote the $u$-step belief state update of $\omega_{i,s}(t)$ when the unobserved arm $i$ is updated for $u$ consecutive slots without being selected. Formally,
\begin{equation}\label{eq:closed}
\resizebox{1\linewidth}{!}{$
\begin{aligned}
        &\omega_{i,s}(u)= \tau^{u-1}(P_{s,1}^{a,i})  \text{ where }\\ &\tau^{u}_i(\omega) =\frac{P_{0,1}^{p,i}-(P_{1,1}^{p,i}-P_{0,1}^{p,i})^u(P_{0,1}^{p,i}-\omega(1+P_{0,1}^{p,i}-P_{1,1}^{p,i}))}{(1+P_{0,1}^{p,i}-P_{1,1}^{p,i})}
\end{aligned}
$}
\end{equation}

This is because $\tau^u(\omega)=\omega \tau^u(1)+(1-\omega) \tau^{u}(0)$, where $\tau^u(1)$ is the $u$-step transition probability from $1$ to $1$ when arm is unobserved, and $\tau^u(0)$ is the $u$-step transition probability from $0$ to $1$ if the transition matrix $P_{s,s^\prime}^{a} = P_{s,s^\prime}^{p} = P_{s,s^\prime}$. From the eigen-decomposition of the transition matrix $\mathcal{P}$, we can have $$\tau^u(1) = \frac{P_{0,1}+(1-P_{1,1})(P_{1,1}-P_{0,1})^u}{1+P_{0,1}-P_{1,1}}$$ 
and 
$$\tau^u(0)=\frac{P_{0,1}(1-(P_{1,1}-P_{0,1})^u)}{1+P_{0,1}-P_{1,1}}$$ 
and solve it to get
\begin{equation*}
\resizebox{1\linewidth}{!}{$
    \tau^u(\omega) = \frac{P_{0,1}-(P_{1,1}-P_{0,1})^u(P_{0,1}-\omega(1+P_{0,1}-P_{1,1}))}{(1+P_{0,1}-P_{1,1})}
    $}
\end{equation*}
However, as the $P_{s,s^\prime}^{a}\neq P_{s,s^\prime}^{p}$, We have $\omega_{s}^i(1)=P_{o_i,1}^{a,i}$, and $\tau^u(\omega)$ as shown in Eq.~\ref{eq:closed}, which leads to $\omega_{s}^i(u)$.

\subsection{Proof of Whittle index decay under partially observable setting (Theorem~\ref{thm:decay})}
\begin{proof}
We can use the induction to prove the index decay. 
Again, $\lambda_0$ satisfies: $V_{\lambda_0,0}(\omega,a=0)=V_{\lambda_0,0}(\omega,a=1)$ for any belief state $\omega\in [0,1 ]$. We can have $\lambda_0+\omega =\omega$, thus $\lambda_0 = 0$. Similarly, $\lambda_1$ can be solved by assuming equation  $V_{\lambda_1,1}(\omega,a=0)=V_{\lambda_1,1}(\omega,a=1)$:
\begin{equation}
\resizebox{\linewidth}{!}{$
\begin{aligned}
        &\omega + \lambda +\beta (\omega P_{1,1}^p + (1-\omega)P_{0,1}^p)= \omega+\beta (\omega P_{1,1}^a+(1-\omega)P_{0,1}^a)\\
        \rightarrow &\lambda =\beta(\omega(P_{1,1}^a-P_{1,1}^p)+(1-\omega)(P_{0,1}^a-P_{0,1}^p)) 
\end{aligned}
$}
\end{equation}
As it is true in the real-world that $P_{s,1}^a>P_{s,1}^p$, thus we have $\lambda_1>0=\lambda_0$. Now we assume the hypothesis that $\lambda_{u}>\lambda_{u-1}$ for $u \in \{1,\dots, T\}$ holds, we must show: $\lambda_{T+1}>\lambda_T$.
There is a similar work done by~\citet{mate2020collapsing}, however, they only show that $\lambda_T>\lambda_1$ for $\forall T > 1$. Their conclusion is built on the fact that for two non-decreasing, linear functions $f_1(\lambda)$ and $f_2(\lambda)$ of $lambda$ and two points $x_1,x_2$. Whenever $f_1(x_1)\leq f_2(x_1)$ and $f_1(x_2)\geq f_2(x_2)$, and if $\frac{\partial f_1}{\partial x}\geq \frac{\partial f_2}{\partial x}$, then $x_1\leq x_2$ is true. 
We here take advantage of this fact and show a different way to demonstrate the index decay in the partially observable setting such that $\lambda_{T+1} > \lambda_{T}$.

We first use $V_{\lambda,T+1}(\omega,a=0)$ to denote function $f_1^p$, and set $f_2^a=V_{\lambda,T+1}(\omega,a=1)$. Then it is obvious that the value of $f_1^p(\omega,\lambda,t)$ will increase as the exogenous reward $\lambda$ increase, whereas the value of $f_2^a(\omega,\lambda,t)$ will increase slower according to the expression. Thus we have the following:
\begin{equation}
\label{eq:decay_1}
   \frac{\partial f_1^p(\omega,\lambda,t)}{\partial \lambda}>\frac{\partial f_2^a(\omega,\lambda,t)}{\partial \lambda} 
\end{equation}
We can prove $\lambda_{T+1} > \lambda_{T}$ through contradiction. We first assume that $\lambda_{T} \geq \lambda_{T-1}$ holds. Then we could have :
\begin{equation}
    f_1^p(\omega,\lambda_{T-1},T-1)=f_2^a(\omega,\lambda_{T-1},T-1);
\end{equation}
and 
\begin{equation}\label{eq:decay_T}
\resizebox{1\linewidth}{!}{$
    f_1^p(\omega,\lambda_{T},T-1)>f_1^p(\omega,\lambda_{T-1},T-1)=f_2^a(\omega,\lambda_{T-1},T-1)>f_2^a(\omega,\lambda_{T},T-1).
    $}
\end{equation}
Similarly, we can also have
\begin{equation}\label{ineq:decay_T}
    f_1^p(\omega,\lambda_{T},T)=f_2^a(\omega,\lambda_{T},T).
\end{equation}. We can get 
\begin{equation}\label{eq:partial_for1}
    \frac{\partial f_1^p(\omega,\lambda_T,t)}{\partial t}<\frac{\partial f_2^a(\omega,\lambda_T,t)}{\partial t}
\end{equation}

However, we assume that $\lambda_{T+1} \leq \lambda_{T}$. 
According to the definition of $\lambda_{T+1}$, we have
\begin{equation}\label{eq:deay_T+1}
    f_1^p(\omega,\lambda_{T+1},T+1)=f_2^a(\omega,\lambda_{T+1},T+1);
\end{equation} 
It is obvious that $\frac{\partial f_1^p(\omega,\lambda,t)}{\partial t}>0$, and $\frac{\partial f_2^a(\omega,\lambda,t)}{\partial t}>0$ (from Lemma~\ref{lem:inc}). According to Eq.~\ref{eq:partial_for1} and Eq.~\ref{ineq:decay_T}, we can have
\begin{equation}\label{ineq:decay_T+1}
    f_1^p(\omega,\lambda_T,T+1)<f_2^a(\omega,\lambda_T,T+1)
\end{equation}
This requires that from $T$ to $T+1$,the Eq.~\ref{eq:partial_for1} also is satisfied. This is also equivalent to show
\begin{equation}\label{eq:cond_decay}
    f_1^p(\omega,\lambda_T,T+1) < f_2^a(\omega,\lambda_T,T)
\end{equation}
Because 
\begin{equation}
\begin{aligned}
    f_1^p(\omega,\lambda_T,T+1) &= \omega + \lambda_T + \beta V_{\lambda_T, T}(\omega_1)\\
    f_2^a(\omega,\lambda_T,T+1) &= \omega + \beta V_{\lambda_T, T}(\omega_2)
\end{aligned}
\end{equation}
where $\omega_1 = P_{1,1}^p \omega + P_{0,1}^p (1-\omega)$ and $\omega_2 = P_{1,1}^a \omega + P_{0,1}^a (1-\omega)$. Because we have $\omega_2>\omega_1$ and $\omega_2>\omega$, thus we can get
\begin{equation}
    V_{\lambda_T, T}(\omega_2) > V_{\lambda_T, T}(\omega_1) \text{ and } V_{\lambda_T, T}(\omega_2)>V_{\lambda_T, T}(\omega)
\end{equation}

According to the definition of $\lambda_T$, we have:
\begin{equation}
    \begin{aligned}
             f_1^p(\omega,\lambda_T,T)&=f_2^a(\omega,\lambda_T,T)\\
             \omega + \lambda_T + \beta V_{\lambda_{T-1}, T-1}(\omega_1) &= \omega + \beta V_{\lambda_{T-1}, T-1}(\omega_2)
    \end{aligned}
\end{equation}
Replace with Eq.~\ref{eq:cond_decay}, we have:
\begin{equation}
\resizebox{\linewidth}{!}{$
\begin{aligned}
    &\beta V_{\lambda_{T-1}, T-1}(\omega_2) -\beta V_{\lambda_{T-1}, T-1}(\omega_1) + \beta V_{\lambda_{T}, T}(\omega_1) < \beta V_{\lambda_{T}, T}(\omega_2) \\
    \rightarrow & V_{\lambda_{T-1}, T-1}(\omega_2) - V_{\lambda_{T}, T}(\omega_2) <  V_{\lambda_{T-1}, T-1}(\omega_1) -  V_{\lambda_{T}, T}(\omega_1)
\end{aligned}
$}
\end{equation}
It is equivalent to show 
\begin{equation}
\begin{aligned}
    &\frac{\partial (V_{\lambda_{T-1}, T-1}(\tau(\omega)) - V_{\lambda_{T}, T}(\tau(\omega)))}{\partial \tau(\omega)}\frac{\partial \tau(\omega)}{\partial \omega} < 0
\end{aligned}
\end{equation}
This is intuitively correct but we leave this as future work to show whether a condition is required to make it always true, as this is not relevant to the fairness constraint in this paper, we just want to show that it is difficult to derive the whittle index value in the finite horizon case.

As $\lambda_{T+1} \leq \lambda_{T}$, according to  Eq.~\ref{eq:deay_T+1} and Eq.~\ref{eq:decay_1}, we can have
\begin{equation}\label{eq:decay_4}
   f_1^p(\omega,\lambda_{T},T+1)\geq f_2^a(\omega,\lambda_{T},T+1) 
\end{equation}

We have Eq.~\ref{eq:decay_4} and Eq.~\ref{ineq:decay_T+1} conflict with each other. Henceforth $\lambda_{T+1}>\lambda_T$.
\end{proof}

\subsection{Proof of Theorem~\ref{thm:q_learning}}\label{ap_subsec:proof_q_learning}
\begin{proof}
We provide our proof which is based on the work by~\citet{biswas2021learn}. Let set $\phi^{\ast}$ to be the set of actions containing the $k$ arms with the highest-ranking values of $Q_i(s,a=1,l)-Q_i(s,a=0,l)$, and any $k$ arms that aren't among the top k are included in the set $\phi^{\prime}$.
Let $\phi^{-,\ast}$ and $\phi^{-,\prime}$ denote the set that includes all of the arms except those in set $\phi^{\ast}$ and $\phi^{\prime}$, respectively. We add the subscript $i$ here in order to avoid ambiguity in the Q-values of distinct arms $i$ at a given state. We could have:
\begin{equation}
\begin{aligned}
    \sum_{i^{\ast}\in \phi^{\ast}}\left[  Q_{i^{\ast}}^{\ast}(s_{i^{\ast}},a_{i^{\ast}}=1,l_{i^{\ast}})-Q_{i^{\ast}}^{\ast}(s_{i^{\ast}},a_{i^{\ast}}=0,l_{i^{\ast}}) \right] \geq \\
    \sum_{j\in \phi^{\prime}}\left[  Q_{j}^{\ast}(s_j,a_j=1,l_j)-Q_{j}^{\ast}(s_j,a_j=0,l_j)\right]
\end{aligned}
\end{equation}
\begin{equation}
    \begin{aligned}
    \sum_{i^{\ast}\in \phi^{\ast}} Q_{i^{\ast}}^{\ast}(s_{i^{\ast}},a_{i^{\ast}}=1,l_{i^{\ast}})+\sum_{j\in \phi^{\prime}}Q_{j}^{\ast}(s_j,a_j=0,l_j) \geq \\
    \sum_{j\in \phi^{\prime}}  Q_{j}^{\ast}(s_j,a_j=1,l_j)+\sum_{i^{\ast}\in \phi^{\ast}} Q_{i^{\ast}}^{\ast}(s_{i^{\ast}},a_{i^{\ast}}=0,l_{i^{\ast}}) 
    \end{aligned}
\end{equation}
Adding $\underset{i\notin \phi^{\ast} \& i\notin \phi^{\prime}}{\sum}  Q_{i}^{\ast}(s_i,a_i=0,l_i)$ on both sides,
\begin{equation}\label{eq:equivalence}
    \begin{aligned}
        \sum_{i^{\ast}\in \phi^{\ast}} Q_{i^{\ast}}^{\ast}(s_{i^{\ast}},a_{i^{\ast}}=1,l_{i^{\ast}})+\sum_{j\in \phi^{-,\ast}}Q_{j}^{\ast}(s_j,a_j=0,l_j) \geq \\
        \sum_{i\in \phi^{\prime}}  Q_{i}^{\ast}(s_i,a_i=1,l_i)+\sum_{j\in \phi^{-,\prime}} Q_{j}^{\ast}(s_{j},a_{j}=0,l_{j})
    \end{aligned}
\end{equation}
Thus from Equation~\ref{eq:equivalence}, we can see that taking intervention action in the action set
As can be seen from Equation~\ref{eq:equivalence}, adopting intervention action for the arms in the set $\phi^{\ast}$ would maximizes $\left\{ \sum_{i=1}^N Q_i^\ast(s,a,l)\right\}$.
\end{proof}

\subsection{Proof of Theorem~\ref{thm:Q_2}}
\label{subsec:thm:Q_2}
\begin{proof}
The key to the convergence is contingent on a particular sequence of episodes observed in the real process~\cite{watkins1992q}. The first condition is easy to be satisfied as to the presence of the fairness constraint. It is a reasonable assumption under the $\epsilon$-greedy action selection mechanism, that any state-action pair can be visited an unlimited number of times as $T\rightarrow \infty$. The second condition has been well-studied in~\cite{hirsch1989convergent,watkins1992q,jaakkola1994convergence}, and it guarantees that when the condition is met, the Q-value converges to the optimal $Q^{\ast}(s,a,l)$. As a result, $Q_i(s,a=1,l)-Q_i(s,a=0,l)$ converges to $Q_i^{\ast}(s,a=1,l)-Q_i^{\ast}(s,a=0,l)$. Also, $Q_i^{\ast}(s,a=1,l)-Q_i^{\ast}(s,a=0,l)$ is the calculated Q-Learning based Whittle index, and choosing top-ranked arms based on these values would lead to an optimal solution.

\end{proof}

\section{General Condition}\label{ap_sec:general_condition}
\subsection{A general condition for Theorem~\ref{thm:FC_infinite_fully}}\label{proof:fully_obs}\label{ap_subsec:general_condition_fully}
During a time interval, we consider the two arms A and B as shown in Figure~\ref{fig:fully_ob_fairness}, B is the k-th ranked arm as shown in Algorithm~\ref{al:alg1}. We assume that the optimal policy is $\pi^\ast$, which only considers the fairness constraint when it is violated at the end of interval, implying that arm $A$ has not been activated in the past $(L-1)$ slots as shown in left. We move the action $a_L=1$ to one slot earlier, where action $a_L$ is to ensure the fairness constraint. We denote this policy as $\pi$ shown in the right. We first consider the fully observable setting here. 
Assume that the state of arm A in time step $L-1$ is $s_A$, and the state of arm B is $s_B$, and $s_A,s_B\in \{0,1 \}$. We let the corresponding Whittle index for arm A is $\lambda_A$, and $\lambda_B$ for arm B. We can calculate the value function $V_{\lambda_i,\infty}(\omega_i)$ for arm A from the time slot $L-1$ under the policy $\pi^\ast$, we have
\begin{equation}\label{eq:fully_1}
\resizebox{\linewidth}{!}{$
    \begin{aligned}
        V_{\lambda_A,\infty}(s_A)=s_A+\lambda_A+\beta\{P_{1,1}^a s_A+P_{0,1}^p(1-s_A)+
        \beta[(P_{1,1}^a s_A\\
        +P_{0,1}^p(1-s_A))V_{\lambda_A,\infty}(P_{1,1}^a) +(1-P_{1,1}^a s_A-P_{0,1}^p(1-s_A))V_{\lambda_A,\infty}(P_{0,1}^{a})  ]   \}
    \end{aligned}
    $}
\end{equation}
Note that $V_{\lambda,\infty}(\omega)=\omega V_{\lambda,\infty}(1)+(1-\omega)V_{\lambda,\infty}(0)$. Similarly, we can get the value function for the arm A under the policy $\pi$:
\begin{equation}\label{eq:fully_2}
\resizebox{\linewidth}{!}{$
    \begin{aligned}
        V_{\lambda_A,\infty}(s_A)=s_A+\beta(s_A V_{\lambda_A,\infty}(P_{1,1}^a)+(1-s_A)V_{\lambda_A,\infty}(P_{0,1}^a))
    \end{aligned}
    $}
\end{equation}
Similarly, the value function for the arm B under policy $\pi^\ast$:
\begin{equation}\label{eq:fully_3}
\resizebox{1\linewidth}{!}{$
    \begin{aligned}
        V_{\lambda_B,\infty}(s_B)=s_B+\beta(s_B V_{\lambda_B,\infty}(P_{1,1}^a)+(1-s_B)V_{\lambda_B,\infty}(P_{0,1}^p)),
    \end{aligned}
    $}
\end{equation}
and the value function for the arm B under policy $\pi$:
\begin{equation}\label{eq:fully_4}
\resizebox{1\linewidth}{!}{$
    \begin{aligned}
        V_{\lambda_B,\infty}(s_B)=\lambda_B+s_B+\beta (V_{\lambda_B,\infty}( s_B P_{1,1}^p+(1-s_B)P_{0,1}^p)),
    \end{aligned}
    $}
\end{equation}
Let the Eq.~\ref{eq:fully_1} minus Eq.~\ref{eq:fully_2} and Eq.~\ref{eq:fully_3} minus Eq.~\ref{eq:fully_4}, and then sum these two values to let it greater than 0, we then get the general condition for theorem~\ref{thm:FC_infinite_fully}. The explicit calculation for the value function can be found in~\citep{liu2010indexability}.

\subsection{General Condition for Theorem~\ref{thm:FC_infinite} and Theorem~\ref{thm:FC_finite}}
\label{ap_subsec:general_condition}
This is similar to the fully observable case~\ref{proof:fully_obs}.
Assume the optimal action set that calculated by Whittle index policy is $A$, and the last several action set that is used to satisfy the fairness constraints is $B$, where we have action $b\in B$ for arm $i$. As shown in Figure~\ref{fig:NIB}.  
\begin{figure}[th]
    \centering
    \includegraphics[width=0.85\linewidth]{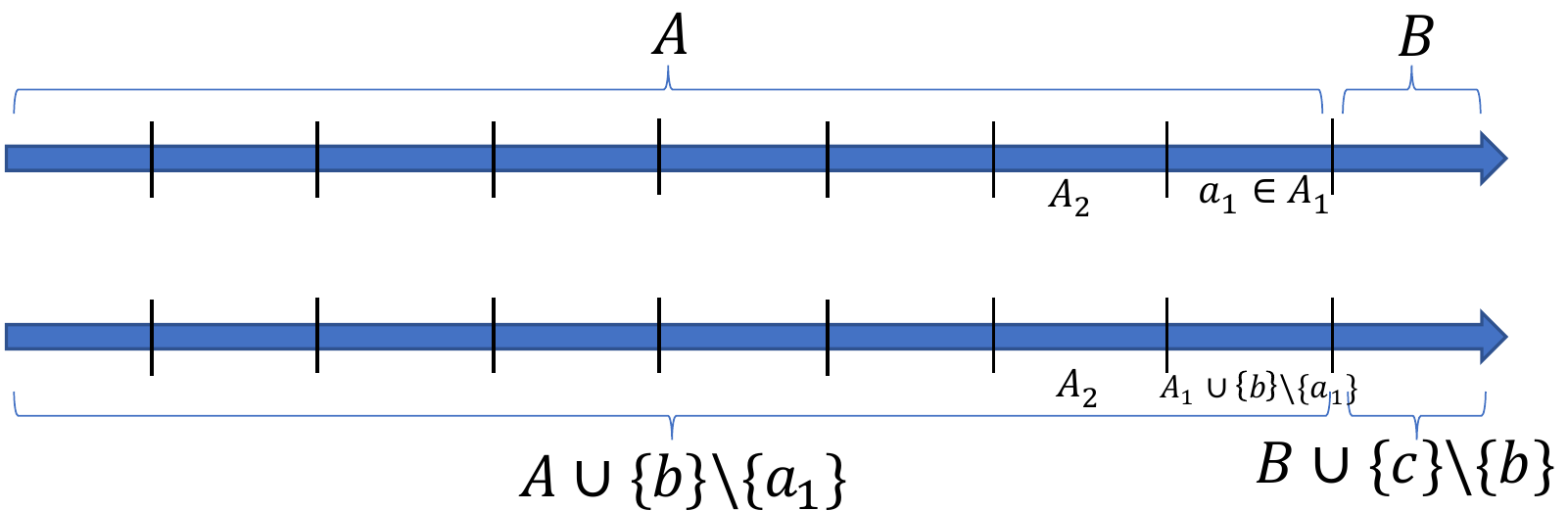}
    \caption{Proof of Theorem~\ref{thm:FC_infinite}}
    \label{fig:NIB}
\end{figure}
If we shift the action $b\in B$ at the end of time slot to one slot earlier, and add another action $c$ for arm $j$ into the action set $B$, then according to the belief state update process, the other following action set except for arm $i$ and $j$ will not be influenced. As a result, we only need to consider arm $i$ and arm $j$. If we showed that the modified policy will not have greater value function. again, then we can repeatedly move the action ${b}$ to one slot earlier and show that  the value function will never have a higher value than the original policy. Hence, we can conclude that the modified policy is smaller than the optimal policy. The Whittle index approach is still applicable under the fairness constraint unless the last few slots are used to ensure that the fairness criteria are met. The numerical condition is as follows,

We consider the partially observable setting. Assume that the probability that the arm A is in good state at time slot $L-1$ is $\omega_i$, and arm B is $\omega_j$, and we assume the corresponding Whittle index for arm A is $\lambda_i$, and $\lambda_j$ for arm B. We can calculate the value function $V_{\lambda_i,\infty}(\omega_i)$ from the time slot $L-1$ for the policy $\pi^\ast$, we have
\begin{equation}\label{eq:eq_pi^ast}
\begin{aligned}
      V_{\lambda_i,\infty}(\omega_i)=\lambda_i+\omega_i+\beta(\omega_i P_{1,1}^p+(1-\omega_i)P_{0,1}^p) \\
    +\beta^2\left\{ (A+B)V_{\lambda_i,\infty}(1) + (C+D)V_{\lambda_i,\infty}(0) \right\}      
\end{aligned}
\end{equation}
where $A=P_{1,1}^p(\omega_i P_{,1}^p+(1-\omega_i)P_{0,1}^p)$, $B=P_{0,1}^a(1-\omega_i P_{1,1}^p-(1-\omega_i)P_{0,1}^p)$, $C=(1-P_{1,1}^a)(\omega_i P_{1,1}^p+(1-\omega_i)P_{0,1}^p)$ and $D=(1-P_{0,1}^a)(1-\omega_i P_{1,1}^p-(1-\omega_i)P_{0,1}^p) $.

Similarly, we can get the value function for the arm A under the policy $\pi$:
\begin{equation}\label{eq:eq_pi}
\begin{aligned}
      V_{\lambda_i,\infty}(\omega_i)=\omega_i+\beta(\omega_i P_{1,1}^a+(1-\omega_i)P_{0,1}^a+\lambda_i) \\
    +\beta^2\left\{ (A+B^\prime)V_{\lambda_i,\infty}(1) + (C^\prime+D^\prime)V_{\infty}(\lambda_i,0) \right\}, 
\end{aligned}
\end{equation}
where $B^\prime=B-P_{0,1}^a+P_{0,1}^p$, $C^\prime=C+\omega_i(P_{1,1}^a-P_{1,1}^p)+(1-\omega_i)(P_{0,1}^a-P_{0,1}^p)$ and $D^\prime=D+\omega_i(P_{1,1}^p-P_{1,1}^a)+(1-\omega_i)(P_{0,1}^p-P_{0,1}^a)+(P_{0,1}^a-P_{0,1}^p)$.
Let Eq.~\ref{eq:eq_pi^ast} minus Eq.~\ref{eq:eq_pi}, we can have
\begin{equation}\label{eq:change_A}
\resizebox{\linewidth}{!}{$
\begin{aligned}
            \lambda_i(1-\beta) + \beta(\omega_i(P_{1,1}^p-P_{1,1}^a)+(1-\omega_i)(P_{0,1}^p-P_{0,1}^a))+\\ \beta^2((P_{0,1}^a-P_{0,1}^p)V_{\lambda_i,\infty}(1)+(P_{0,1}^p-P_{0,1}^a)V_{\lambda_i,\infty}(0))
\end{aligned}
$}
\end{equation}
Similarly, for the arm B, we can get the change in the value function from the policy $\pi^\ast$ to $\pi$:
\begin{equation}\label{eq:change_B}
\resizebox{\linewidth}{!}{$
\begin{aligned}
         \lambda_j(\beta-1) - \beta(\omega_j(P_{1,1}^p-P_{1,1}^a)-(1-\omega_j)(P_{0,1}^p-P_{0,1}^a))+\\ \beta^2((P_{0,1}^a-P_{0,1}^p)V_{\lambda_j,\infty}(1)+(P_{0,1}^p-P_{0,1}^a)V_{\lambda_j,\infty}(0))   
\end{aligned}
$}
\end{equation}
We sum Eq.~\ref{eq:change_A} and Eq.~\ref{eq:change_B} and let it greater than $0$, then this is the condition that the optimal policy of theorem~\ref{thm:FC_infinite} for RMAB with the fairness constraint is to select the arm to play when the fairness constraint is violated at the time interval under the partially observable setting. Explicit computation of the value function can be found in~\citep{liu2010indexability}.

\begin{figure*}[ht]
  \centering
  \begin{minipage}[b]{\linewidth}
  \includegraphics[width=0.33\linewidth]{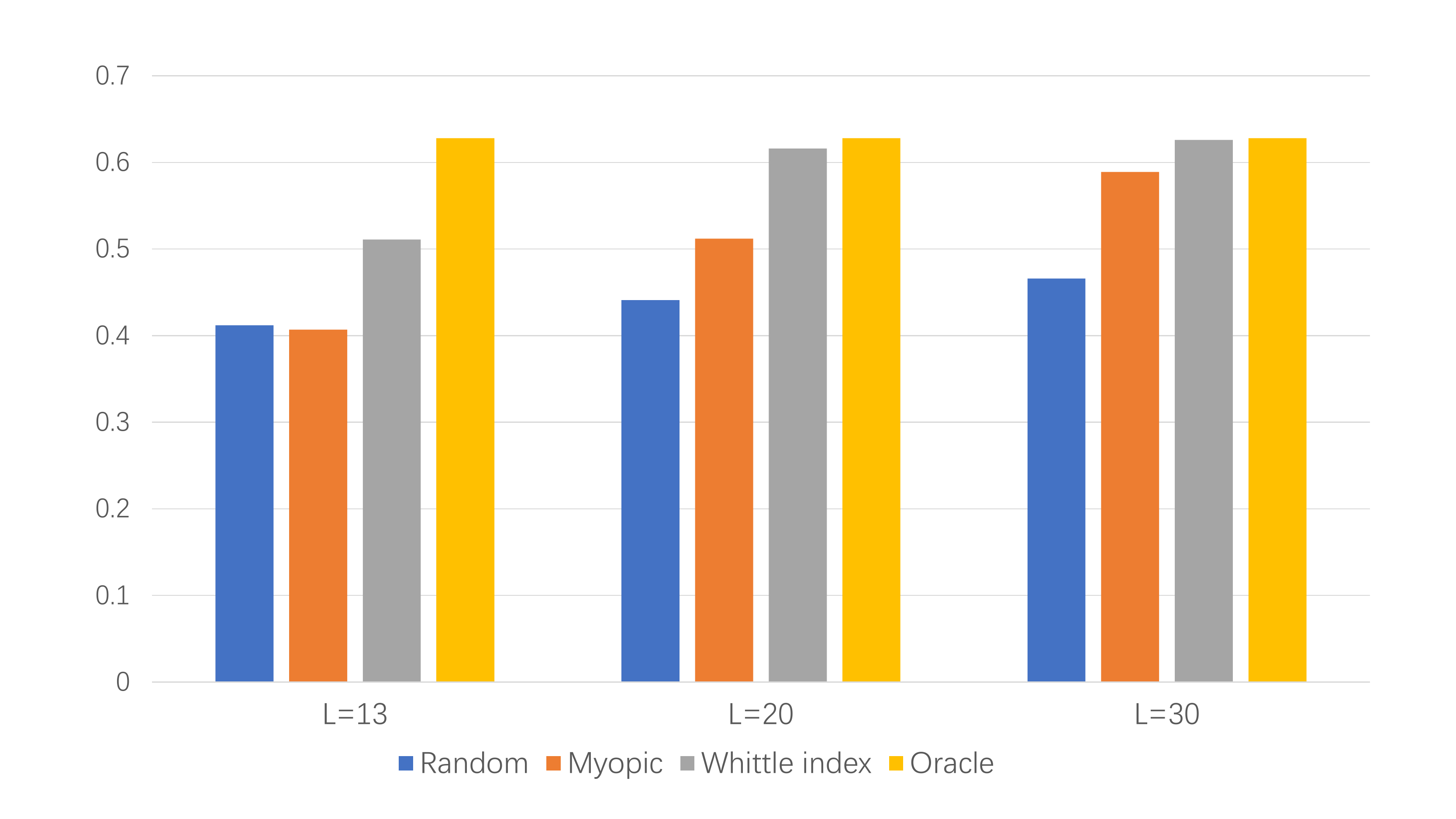} 
  \includegraphics[width=0.33\linewidth]{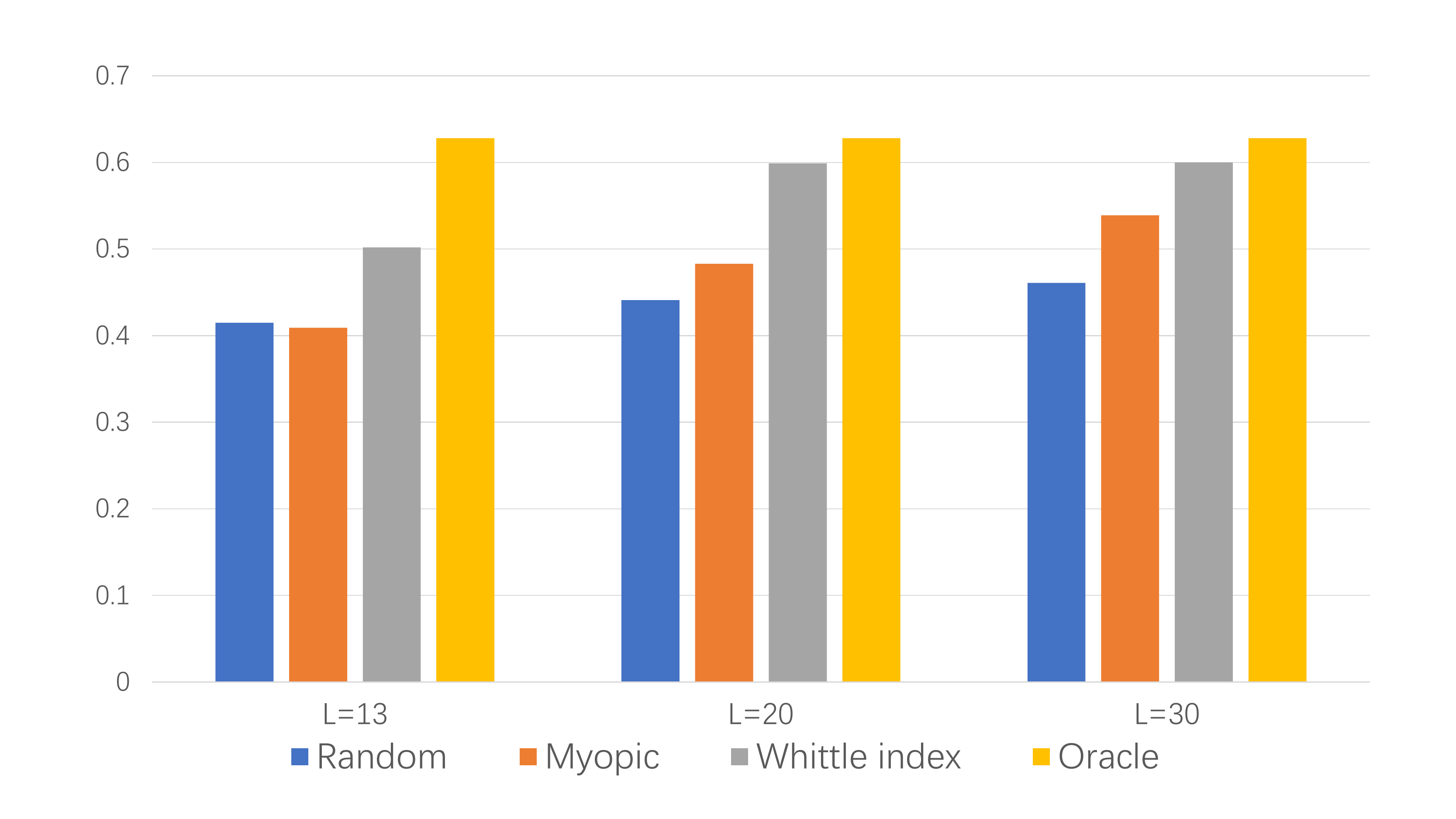} 
  \includegraphics[width=0.33\linewidth]{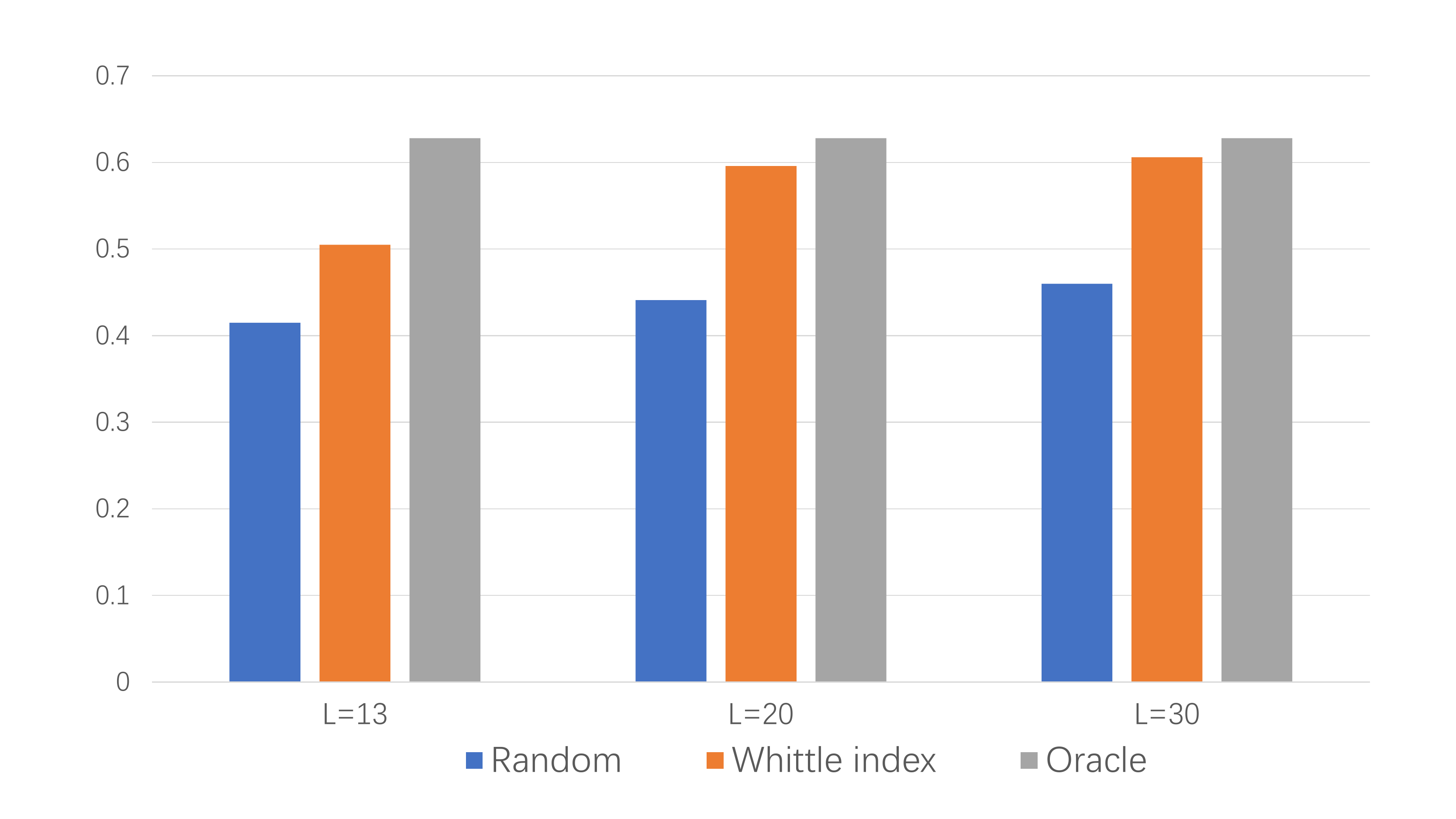}
  \end{minipage}
  \caption{The average reward of each arm over the time length $T=1000$ with small penalty for the violation of the fairness constraint.}
  \label{fig:sensity}
\end{figure*}

\begin{figure*}[ht]
  \centering
  \begin{minipage}[b]{\linewidth}
  \includegraphics[width=0.33\linewidth]{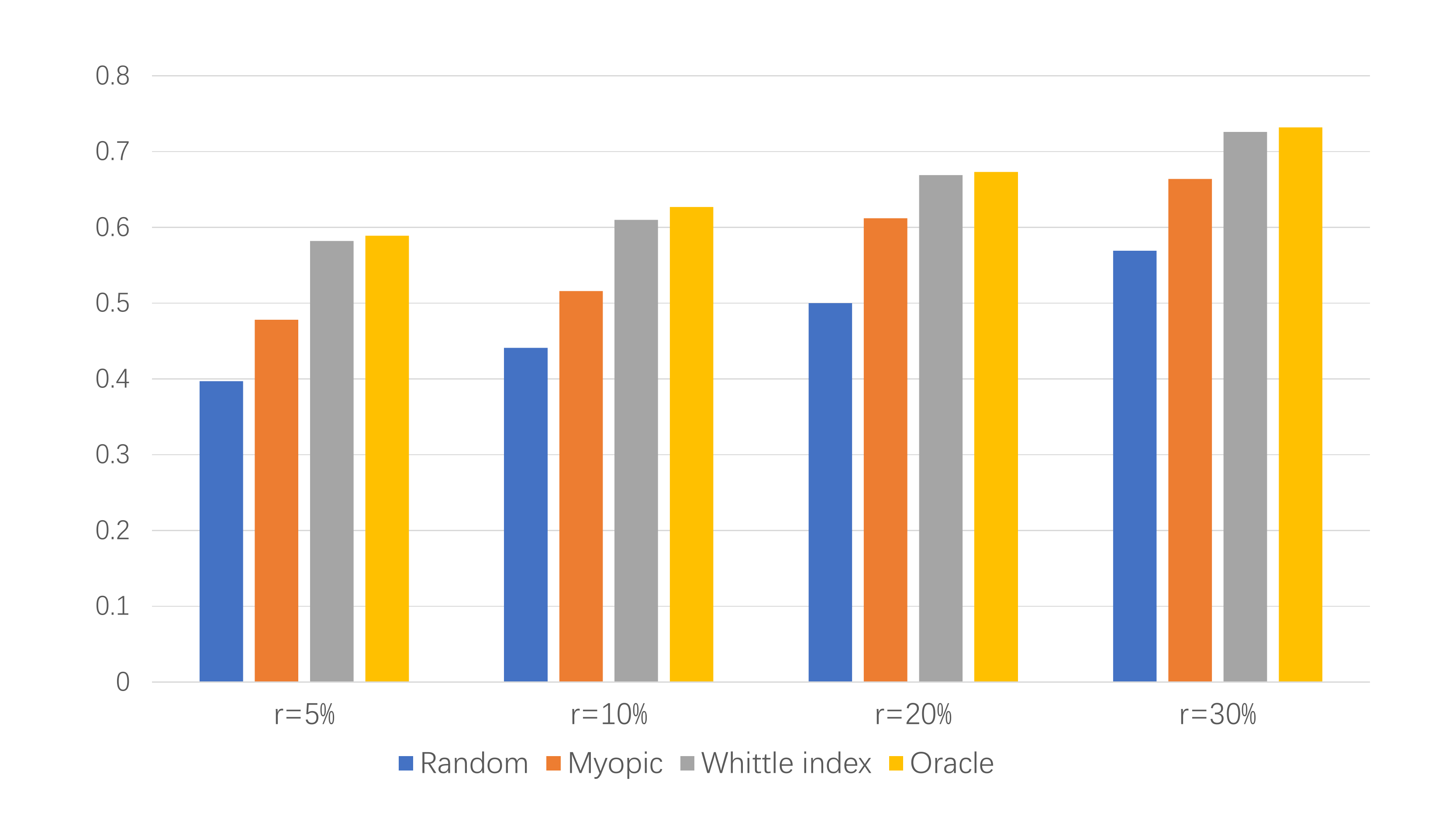} 
  \includegraphics[width=0.33\linewidth]{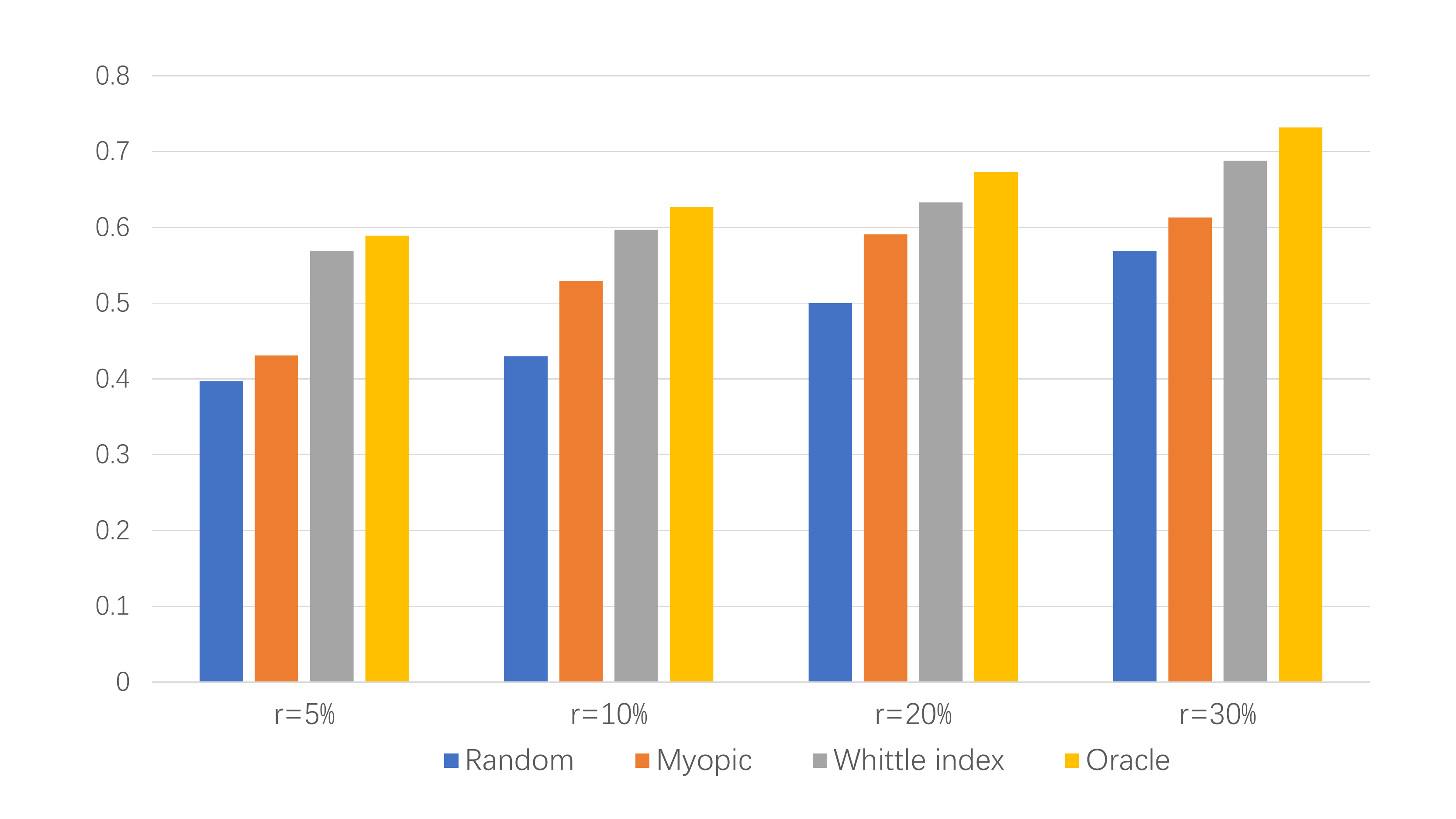} 
  \includegraphics[width=0.33\linewidth]{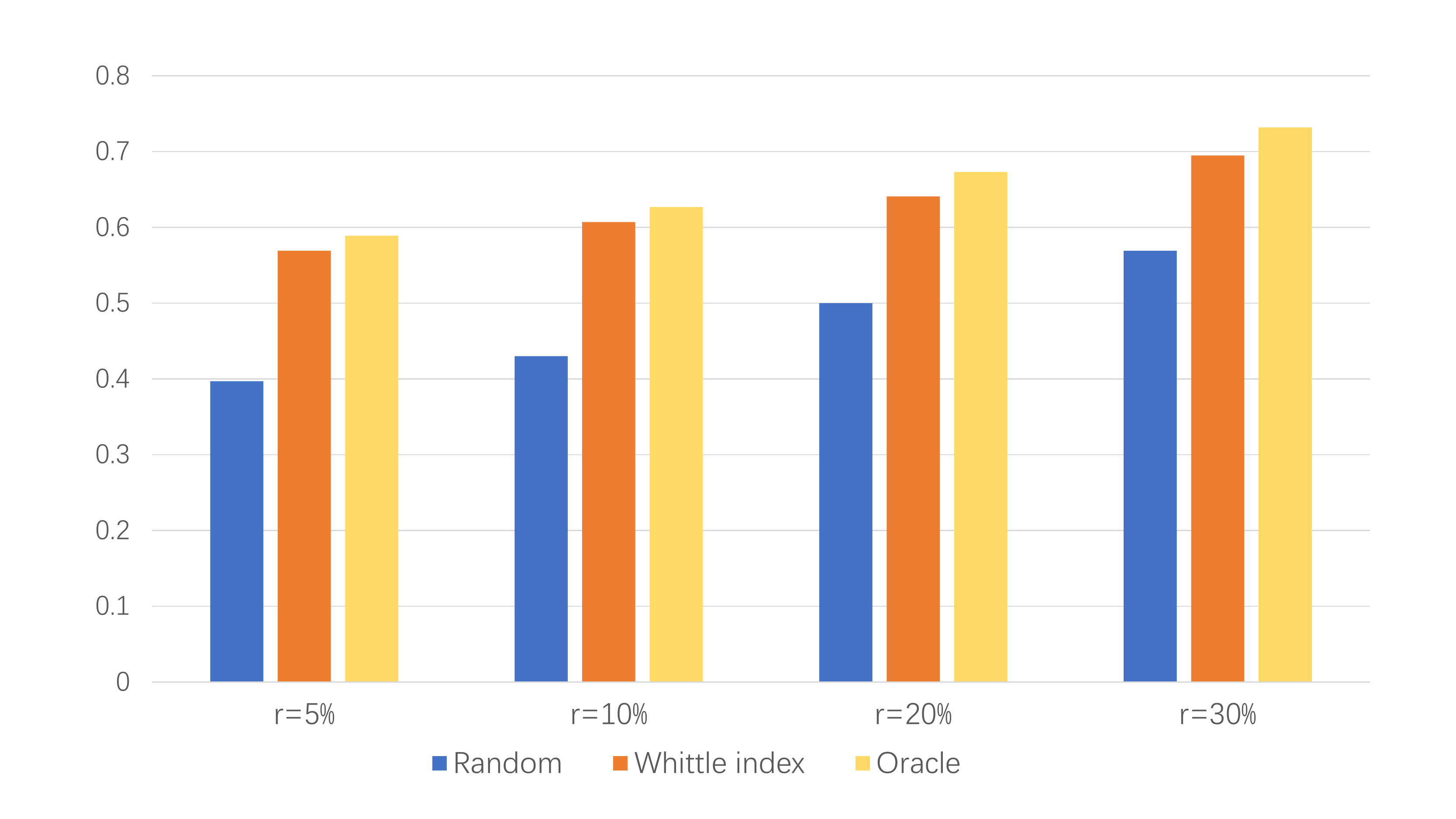}
  \end{minipage}
  \caption{The average reward of each arm over the time length $T=1000$ with small penalty for the violation of the fairness constraint.}
  \label{fig:k_level}
\end{figure*}

\section{Visualization of algorithm}\label{app:vis}
We give a visualization of our proposed Whittle index based approach to solve the fairness constraint in Figure~\ref{fig:Belief_vis}.
\begin{figure}[ht]
    \centering
    \includegraphics[width=0.98\linewidth]{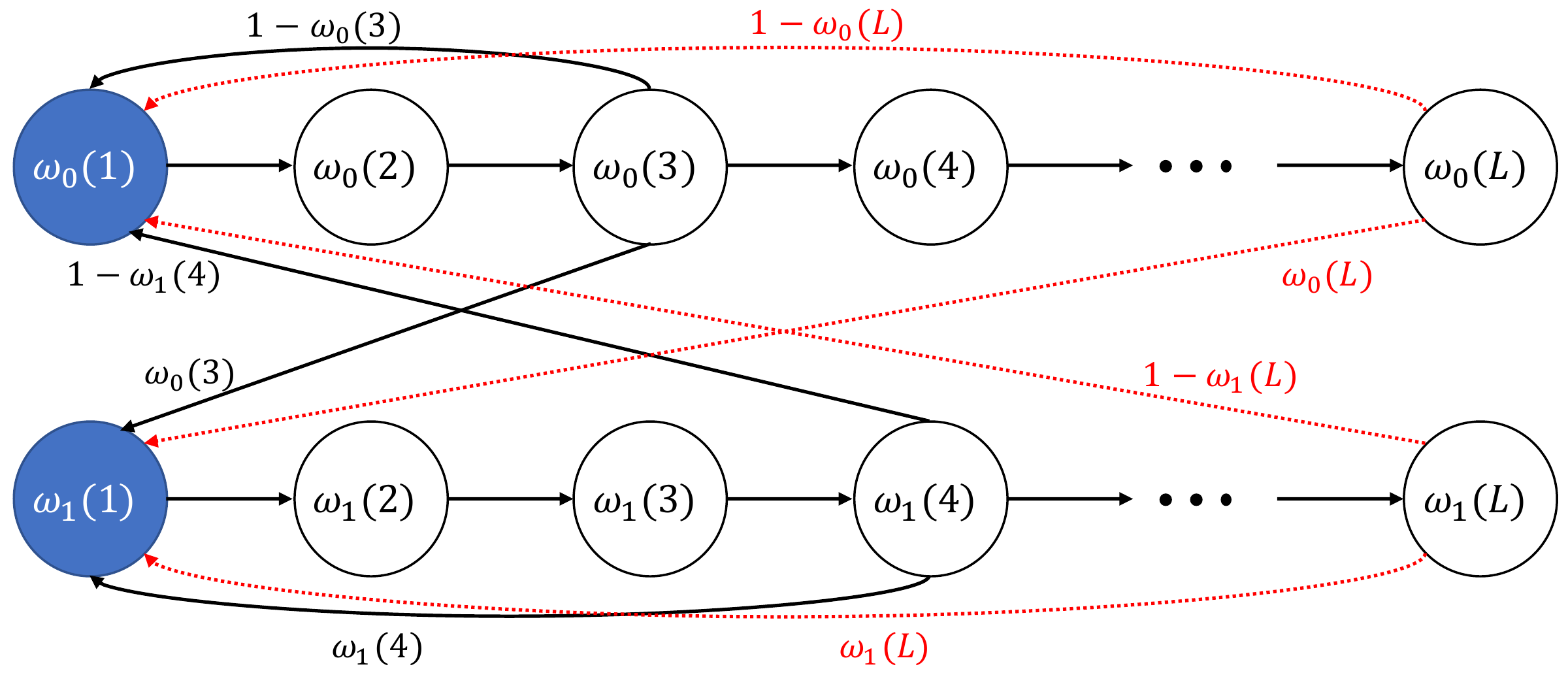}
    \caption{Visualization of Whittle index approach with fairness constraints.}
    \label{fig:Belief_vis}
\end{figure}

The belief state MDP works as follows: initially, after an action, the state $s\in\{ 0,1\}$ of the selected arm is observed. Then the belief state then changes to $P_{s,1}^a$ one slot later, which is represented as the blue node at the head of the chain. Subsequent passive actions cause the belief state evolves according to the initial observation in the same chain. Then if the arm is activated again under the proposed algorithm, it will transit to the head of one of the chains with the probability according to its belief state as shown in the black arrow. If the arm's fairness constraint is not met, i.e., it has not been chosen in the last $L-1$ time slots, it will be activated at the time slot $L$, and go to the head of one of the chains (as shown by the red dashed arrow).

\section{Additional Results}
\label{app:additional}
\subsection{Fairness Constraints Strength}
In this section, we provide the average reward results for different fairness constraints, and see how they influence the overall performance. 
The strength of fairness restrictions is represented by the combination of $L$ and $\eta$.
For instance, $\eta$ is a parameter to determine the lowest bound of the number of times an arm should be activated in a decision period of length $L$. Smaller $L$, on the other hand, indicates that a strict fairness constraint should be addressed in a shorter time length. For ease of explanation, we fix the value of $\eta=2$, which means that an arm will be activated twice in any given time steps of length $L$. We can change the value of $L$ to measure the fairness constraint level.
We investigate three different categories of fairness constraint strength as follows,
\begin{itemize}
    \item \textbf{Strong level}: The strong fairness constraints impose a strict restrictions on the action. Here we assume that the strong fairness constraints $L$ satisfy $\frac{k\times L}{N}=1.3$, this can translate to at most $30\%$ arms can be engaged twice when before all arms have been pulled previously.
    \item \textbf{Medium level}: We define the medium fairness constraints by solving: $\frac{k\times L}{N}=2$.
    \item \textbf{Low level}: The low strength of fairness constraints can be interpreted as a low fairness restriction on the distribution of the resources, i.e., we have $\frac{k\times L}{N}=3$, which means all arms will receive the health intervention before each arm has been activated three times on average.
\end{itemize}
We provide the average reward results in Figure~\ref{fig:sensity}. Again, the left graph shows the performance of Whittle index approach with fairness constraint when the transition model is known, the middle graph presents the result of the Thompson sampling-based approach for Whittle index calculation, the the right graph shows the result for the Q-Learning based Whittle index approach. Our proposed approach can handle fairness constraints at different strength level without sacrificing significantly on the solution quality.

\subsection{Intervention Level}
In this part, we present the performance results for various resource levels where the fairness constraint $L$ is fixed, and we ensure that $k \times L < N$. Here, we let $L=30$, and $N=100$, and we're looking at the performance of the intervention ratio where $\frac{k}{N}=5\%, 10\%, 20\%, 30\%$ respectively. We can see that our proposed approach to solve the fairness constraint can consistently outperform the Random and Myopic baselines regardless of the intervention strength while does not have significant differences when compared to the optimal value without taking fairness constraints into account.



\end{document}